\documentclass{article}
\usepackage{ijcai26}

\usepackage{times}
\usepackage{soul}
\usepackage{url}
\usepackage[hidelinks]{hyperref}
\usepackage[utf8]{inputenc}
\usepackage[small]{caption}
\usepackage{graphicx}
\usepackage{amssymb,amsmath,amsfonts}
\usepackage{amsthm}
\usepackage{booktabs}
\usepackage{algorithm}
\usepackage{algorithmic}
\usepackage[switch]{lineno}
\usepackage{booktabs}
\usepackage{multirow}
\usepackage{subfigure}
\usepackage{makecell}
\usepackage{enumerate}

\newtheorem{lemma}{Lemma}

\title{Strikingness-Aware Evaluation for Temporal Knowledge Graph Reasoning}

\author{
Rikui Huang$^{1,2,3}$
\and
Shengzhe Zhang$^2$\And
Wei Wei$^{1,2,3,}$\thanks{Corresponding author}
\affiliations
$^1$School of Computer Science \& Technology, Huazhong University of Science and Technology \\
$^2$Institute of Artificial Intelligence, Huazhong University of Science and Technology\\
$^3$School of Artificial Intelligence \& Automation, Huazhong University of Science and Technology\\
\emails
\{huangrk, zsz, weiw\}@hust.edu.com
}

\begin{document}

\maketitle

\begin{abstract}
Temporal Knowledge Graph Reasoning (TKGR) aims at inferring missing (especially future) events from historical data. Current evaluation in TKGR uniformly weights all events, ignoring that most are trivial repetitions, which overestimate the true reasoning ability. Therefore, the rare outstanding events, whose prediction demands deeper reasoning, should be distinguished and emphasized. To this end, we propose a strikingness-aware evaluation framework, which introduces a rule-based strikingness measuring framework (RSMF) to quantify event strikingness by comparing its expected occurrence with peer events derived from temporal rules. Strikingness is then integrated as a weighting factor into metrics like weighted MRR and Hits@k.
Experiments on four TKG benchmarks reveal: 1) All representative models perform worse as event strikingness increases, 2) Path-based methods excel on low-strikingness events and representation-based ones on high-strikingness events, 3) We design an ensemble method whose gains stem from fitting trivial events rather than reasoning improvement. 
Our framework provides a more rigorous evaluation, refocusing the field on predicting outstanding events.
\end{abstract}

\section{Introduction}
Recent advances in Temporal Knowledge Graph Reasoning (TKGR) have led to substantial progress, which can be broadly categorized into two classes according to whether they forecast future events: interpolation and extrapolation reasoning ~\cite{jin2020recurrent}. The former refers to inferring missing historical facts, while the latter involves predicting future events, also named temporal knowledge graph forecasting \cite{TITer}. This work primarily focuses on extrapolation reasoning, which is essential for many high-risk applications like financial risk control~\cite{aven2013meaning}.

Despite promising empirical results~\cite{liang2024survey}, many reported gains may largely arise from data biases, leading to misjudgment of advancements in the field~\cite{KervadecAB021}. A historical parallel exists in the static KGR field, potential data leakage in well-known benchmarks (WN18, FB15k) led to an overestimation of the reasoning capabilities of models \cite{Kristina2015,dettmers2018convolutional}. Over 94\% and 81\% of queries, such as (A, \textit{hypernym}, ?), in WN18 and FB15k can easily be mapped to a training triple (B, \textit{hyponym}, A) if it is known that hyponym is the inverse of hypernym. Recently, a similar phenomenon is emerging in the TKGR field: over 80\% of events have occurred in prior history in the ICEWS data \cite{CyG-Net}. It may enable heuristic-based predictions, inflating state-of-the-art (SOTA) performance, under the existing TKGR evaluation framework, on common events while obscuring poor accuracy on fewer than 10\% of truly challenging, striking cases. For example, given a query such as $(A, MakeVisit, ?, T_q)$, they may output an answer $B$ by selecting either the most frequently occurring historical event $(A, MakeVisit, B, T_i)$~\cite{ICL,CENET}. It raises doubts about the predictive quality of current TKGR methods and whether the current evaluation framework could reasonably reflect these models' forecasting capabilities \cite{Recurrency}.

The above flaw in static KGR has been addressed by removing inverse relation triples from their training sets, i.e., creating WN18RR \cite{dettmers2018convolutional} and FB15k-237 \cite{Kristina2015}. However, this straightforward removal strategy is fundamentally inapplicable to TKGR, since all historical events, even repetitive ones, constitute essential evidence for forecasting the future. This dilemma raises a critical question: How to construct a more meaningful evaluation framework for TKGR without deleting data?

A principled feasible alternative is to re-weight test-instance gains, rather than treating all instances uniformly. Specifically, trivial events such as (A, \textit{MakeVisit}, B) occurring across different timestamps frequently should be assigned lower weights. In contrast, rarer outstanding events like (A, \textit{Sign Agreement}, B), which require deeper temporal reasoning, should be emphasized. Generally, accurately inferring outstanding events offers far greater practical value than merely predicting numerous trivial ones. However, measuring the weights and automatically identifying outstanding ones from a large volume of trivial events is nontrivial. While there have been some efforts to measure strikingness of facts in static KGs, there is a notable lack of studies in TKGs field. This gap stems from two challenges: First, beyond the statistical features, a comprehensive TKGR evaluation framework necessitates incorporating both semantic and temporal relevance. Second, since the ground-truth impact of a future event is unknowable in advance, any measure of its strikingness can only be derived from observable historical patterns.

To address this, we propose a \textbf{R}ule-based \textbf{S}trikingness \textbf{M}easuring \textbf{F}ramework (\textbf{RSMF}) to measure the strikingness of future events based on historical evidence. RSMF first leverages first-order temporal rules to retrieve peer events for the target event. Subsequently, it computes the expected occurrence of candidate events with the semantic confidence of rules, the temporal characteristics of events, and the frequency of event repetition. The strikingness of the future event is derived by contrasting its expected occurrence with that of its peer events. Finally, we construct a strikingness-aware evaluation framework by using strikingness as a weighting factor. Experimentally, we evaluate eight representative baselines under the striking-aware evaluation framework across four widely adopted TKG datasets, including three path-based, three representation-based, and two large language model (LLM)-based approaches. Our contributions and key findings can be summarized as follows:
\begin{itemize}
\item We propose RSMF to quantify event strikingness in TKGs, and build a new corresponding striking-aware TKGR evaluation framework that re-weights test instances with their strikingness.
\item For evaluated models, reasoning performance decreases as event strikingness increases, i.e., events with higher strikingness are more difficult to predict.
\item We find distinct performance patterns across baselines: path-based methods show stronger performance in low-strikingness events, whereas representation-based approaches excel at high-strikingness events.
\item We design an ensemble method combining path- and representation-based models, aiming to leverage their complementary strengths. Consequently, it separately obtains significant and marginal gains in the original and our proposed strikingness-aware framework. Analysis reveals that while the method's gains come from dominant low-strikingness events, whereas, performance on rare high-strikingness events decreases.
\end{itemize}

\section{Related Work}
\paragraph{Temporal Knowledge Graph Reasoning Evaluation} In recent years, researchers have proposed various extrapolation TKGR methods, including graph neural networks-based \cite{li2021temporal,TiRGN,LogCL}, rule-based \cite{liu2022tlogic,TempValid}, reinforcement learning based \cite{TITer,zheng2023dream,DaeMon}, and the increasingly popular large language models-based methods \cite{ICL,liao2024gentkg,xia2024chain}. Alongside these advancements, the common rank-based evaluation, for link prediction like KGR, methods are undergoing continuous refinement. Initially, to address the issue of multiple answers for a single query, correct answers other than the target answer are filtered out during ranking to avoid underestimating model performance \cite{TransE}. Subsequently, to accommodate TKGR, time-aware filtering \cite{xERTE} and time interval prediction evaluation were introduced \cite{jain2020temporal}. Considerable efforts have also been made in re-evaluating the performances of various models to establish fair comparisons \cite{sun2020re,RuffinelliBG20you,TKGFEvaluation}. In addition, a valuable baseline named \textit{Recurrency} highlight flaws in datasets and offered significant insights \cite{Recurrency}. To explore the capability boundaries of TKGR models, some studies attempt to build new benchmark datasets tailored to different scenarios, such as context-aware \cite{ma2023context}, multi-modal \cite{li2024mm}, and large-scale settings \cite{gastinger2024tgb}. However, they may also suffer from the above data biases, as the repetitive pattern is an inherent characteristic of TKGs. 

\paragraph{Outstanding Facts Mining in Knowledge Graph}
Outstanding facts (OFs) mining focuses on quantifying the strikingness of facts. Early research focused on extracting OFs from unstructured data (e.g., text) \cite{angiulli2009detecting,hassan2014data,wu2012one}. Maverick~\cite{zhang2018maverick} firstly measured event strikingness in static KGs with the specific attribute values of entities. FMINER~\cite{yang2021context} introduced context entity constraints and designed a pattern relevance model to optimize the process of event searching. The robustness of measured outstanding events is further explored using perturbation analysis ~\cite{xiao2024avoid}. To the best of our knowledge, our framework is the first principled extension of the established Outstanding Fact Mining paradigm from static KGs to TKGs. Furthermore, we transform a mining technique into a comprehensive evaluation framework, creating weighted metrics that reorient the TKGR field towards valuing outstanding reasoning.

\section{Strikingness-Aware Evaluation}
\subsection{Preliminaries}
\paragraph{Temporal Knowledge Graph Reasoning} A TKG can be represented as a sequence of timestamp KGs, denoted as $\mathcal{G}=\{\mathcal{G}_1, \mathcal{G}_2, $ $\ldots, \mathcal{G}_t\}$. Each KG at a specific timestamp $t$ is defined as $\mathcal{G}_t=\left(\mathcal{E}, \mathcal{R}, \mathcal{F}_t\right)$, where $\mathcal{E}$ is the set of entities, $\mathcal{R}$ represents the set of relations, and $\mathcal{F}_t = \{(s, r, o, t) \}$ refers to the set of events observed at timestamp $t$. Given a query $(s, r, ?, t)$, a reasonable TKGR model is to infer the object $o$ based on the facts observed before $t$, where $s$ and $o$ are subject and object entities, $r$ is a relation, and $t$ is a timestamp. For instance, the query  (\textit{Markieff Morris}, \textit{join}, ?, 2025-02) requires the model to predict \textit{the Lakers} based on events before 2025-02 to validate its forecasting capability in practice. 

\begin{figure*}[ht]
    \centering
    \includegraphics[width=\textwidth]{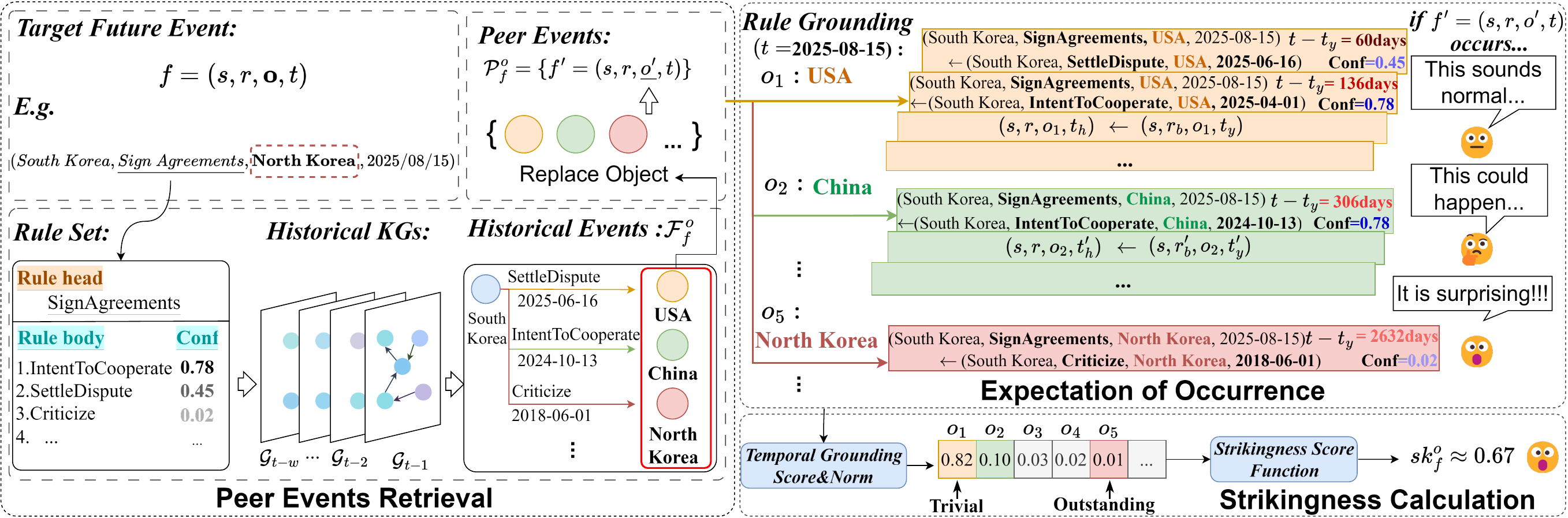}
    \caption{\normalfont An example of strikingness measuring for target future event $(South\ Korea, Sign\ Agreements, North\ Korea, 2025/08/15)$ (replacing object). In \textbf{Peer Events Retrieval}, RMFS retrieves the historical events and constructs the peer events with the rule set. The \textbf{Expectation of Occurrence} and \textbf{Strikingness Calculation} are calculated by the rule grounding and strikingness scoring function.}
    \label{figure:FOE Measuring}
\end{figure*}
\paragraph{Strikingness of Events} Strikingness quantifies how outstanding a target event $f=(s, r, o, t)$ is, which is a continuous value in the range [0, 1]. The closer the value is to 1, the more outstanding the event, and vice versa. Events with low strikingness can be referred to as \textit{trivial events}, while events with high strikingness can be referred to as \textit{outstanding events}. Since it is not meaningful to compare two entirely unrelated events, such as \textit{Markieff Morris will join the Lakers in 2025} and \textit{The Federal Reserve will cut interest rates in 2026}, the strikingness thus is defined by comparing it with peer events $\mathcal{P}$. 

\paragraph{Peer Events} A peer event is a related event of the target event generated by replacing entities or relations. For a target future event $f=(s, r, o, t)$, its peer events are defined as $\mathcal{P}_{f}^{s}=\{(s',r,o,t)|s'\in\mathcal{E}\}$, $\mathcal{P}_{f}^{o}=\{(s,r,o',t)|o'\in\mathcal{E}\}$, and $\mathcal{P}_{f}^{r}=\{(s,r',o,t)|r'\in\mathcal{R}\}$. 

\subsection{Strikingness Measuring}
To measure the strikingness of an event, three key challenges must be addressed: 1) Constructing a set of peer events that can be used for comparison with the target future event, 2) Computing the expectation of occurrence of the target event and its peer events, and 3) Calculating the strikingness score of the target event using the expectation of occurrence. The overall procedure for RSMF is outlined in Figure \ref{figure:FOE Measuring}.

\paragraph{Peer Events Retrieval} Peer events can be obtained by replacing the entities or relations of the target event. However, direct substitution may generate many meaningless peer events, such as (\textit{Markieff  Morris}, \textit{join}, \textit{Microsoft}, 2025-02). Therefore, we utilize temporal rules to constrain the generation of peer events from historical KGs. 

For a target future event $f=(s,r,o,t)$, we first obtain the rule set $TR$ corresponding to the relation $r$ through rule mining \cite{liu2022tlogic}. While higher-order rules could capture more complex patterns, they also introduce exponential computational complexity and risk of overfitting. Thus, we only use the length 1 rules as a practical measure. A detailed complexity analysis is provided in Appendix \ref{app:complexity}. A temporal rule is defined as follows:
\begin{align}
\label{temporal_rule}
(E_1, r_h, E_2, T_2) \leftarrow (E_1, r_b, E_2, T_1)
\end{align}

\noindent where $T_{1} < T_{2}$, $r_h$ and $r_b$ denote rule head and body relation, $E_i$ and $T_i$ indicate entity and timestamp variables.

For ease of understanding, we take the retrieval of the object entity as an example. We first mask the object entity in the target event to convert it into a query $f_q = (s, r, ?, t)$, and then use the historical KG sequences and temporal rules to search for historical events that support this query. For a rule $tr \in TR$, we ground the rule body in the historical KG sequences $\left\{\mathcal{G}_i\right\}_{i=t-w}^{t-1}$ to obtain the grounded historical events:
\begin{align}
\label{grounding_events}
\mathcal{F}_{f,tr}^{o}=\{(s,r_{b},o',t')|t-w\leq t'<t\}
\end{align}
\noindent where $r_{b}$ represents the relation of the rule body for $tr$, while $w$ controls the window of the historical KG sequences. The grounded events of all rule bodies for $tr$:
\begin{align}
\mathcal{F}_{f}^{o}=\bigcup_{tr \in TR} \mathcal{F}_{f,tr}^{o}
\end{align}
Then, we take the object entities in $\mathcal{F}_{f}^{o}$ as the candidate set $\mathcal{C}_{f}^{o} = \{o'|(s,r_{b},o',t') \in \mathcal{F}_{f}^{o} \}$ for object entity replacement. Further, peer object events $\mathcal{P}_{f}^{o} = \{(s,r,o',t)| o'\in \mathcal{C}_{f}^{o} \}$ of the target event $f$ can be generated by substituting the object entities $o$. The peer relation events $\mathcal{P}_{f}^{r}$ and the peer subject events $\mathcal{P}_{f}^{s}$ could be obtained similarly.

\paragraph{Expectation of Occurrence} Strikingness is related to human expectations regarding the occurrence of an event. To this end, we employ a rule-based approach to compute the expected scores of the target event and its peer events. For a peer event $f'$, we ground the instances by applying each rule $tr \in TR$ to obtain the rule grounding events pair:
\begin{align}
rg_{f'}^{tr}=(s,r,o',t_{h})\leftarrow(s,r_{b},o',t_{y})
\end{align}
Each $rg_{f'}^{tr}$ consists of a rule body and a rule head means the body event could support the occurrence of head event.

Furthermore, we also consider the impact of event frequency on strikingness measurement. Intuitively, the more frequently the rule grounding observed, the higher the expectation of event $f'$. Therefore, we iteratively collect rule grounding to obtain the set of rule grounding:
\begin{align}
\label{rule_grounding_set}
RG_{f'}^{tr}=\{(s,r,o',t_{h_i})\leftarrow(s,r_{b},o',t_{y_j})\}_{j=1}^{n},
\notag \\ 
with\ t_{y_1}<t_{h_1}\leq t_{y_2}<t_{h_2}\leq...\leq t_{y_n}<t_{h_n}
\end{align}

\noindent where $t_{y_1} \ge t-w$, and $n$ represents the number of rule groundings.  The temporal constraint is utilized to avoid overlap and redundancy. Additionally, we set $t_{h_n} = t$, indicating that in the temporally closest grounding, only the rule body is a historical event, which provides support for reasoning about the potential future event $f'$. 

After obtaining the rule grounding set, we compute the expectation score of target event and its peer events. Following the rule-based reasoning methods \cite{ott2023rule,TempValid}, we design the expectation score from two aspects. First, the effectiveness of rule-based reasoning will decay over time, and we use an exponential distribution to model this decay. Second, since different rules contribute variably to expectation, we employ the confidence of each rule to reflect its contribution. The expectation score is calculated as follows:
\begin{align}
\label{expectation_score}
sc_{f'}=\sum_{tr\in TR}\sum_{rg_{f'}^{tr}\in RG_{f'}^{tr}}conf(tr)*e^{-\lambda(t-t_y)}
\end{align}
where $conf(tr)$ represents the confidence of rule $tr$, $\lambda > 0$ is the temporal decay coefficient, and $t_y$ denotes the timestamp of the rule body in the rule grounding $rg_{f'}^{tr}$. 

According to the different elements being replaced, we obtain the corresponding score sets of the target event $f$ and transform them to vectors, denoted as $\mathbf{sc}_{f}^{s}\in \mathbb{R}^{|\mathcal{C}_{f}^{s}|}$, $\mathbf{sc}^{r}\in \mathbb{R}^{|\mathcal{C}_{f}^{r}|}$, and $\mathbf{sc}^{o}\in \mathbb{R}^{|\mathcal{C}_{f}^{o}|}$. These three sets of scores estimate the expectation of events from different perspectives. 

\paragraph{Strikingness Calculation} The expectation score reflects the prior perception of the likelihood of an event’s occurrence. Strikingness measures the degree to which the event exceeds the prior expectation and is thus inversely correlated with the expectation score. That is, the higher the expectation score assigned to the event $f'$, the less striking the occurrence of the event. To constrain strikingness within the range [0, 1], the obtained score sets need to be normalized as follows:
\begin{align}
\label{sc_norm}
\mathbf{sc}_{norm}^{be} = \mathbf{sc}_{f}^{be}/||\mathbf{sc}_{f}^{be}||_2
\end{align}
where body element $be \in \{ s,r,o\}$ is an replaced element, and $||\cdot||_2$ is L2 normalization.

Then, we compare the normalized scores of peer events $f'$ to highlight the prominence of the target event $f$. The strikingness scoring function accounts for both the magnitude of peer event scores and the differences between them. Based on the strikingness measure proposed in ~\cite{angiulli2009detecting}, we adopt the following function to calculate the strikingness of the body elements:

\begin{align}
\label{body_strikingness}
sk_{f}^{be}=\sum sc_{f'}^{be}*(sc_{f'}^{be}-sc_{f}^{be})*\mathbb{I}(sc_{f'}^{be}>sc_{f}^{be})
\end{align}

\noindent where $sc_{f'}^{be}\in \mathbf{sc}_{norm}^{be}$, and $\mathbb{I}(\cdot)$ is the indicator function that returns the value 1 if the condition is true and 0 otherwise.

Finally, we weight the strikingness of all body elements to obtain the final strikingness of potential future event $f$:
\begin{align}
\label{event_strikingness}
sk_{f}&=\alpha^{s}sk_{f}^{s}+\alpha^{o}sk_{f}^{o}+\alpha^{r}sk_{f}^{r}
\end{align}
where $\alpha^{s};\alpha^{o};\alpha^{r}\in [0,1]$ are weights of body elements and $\alpha^{s}+\alpha^{o}+\alpha^{r}=1$. We provide the proof for the bound $sk_f \in [0, 1]$ in Appendix \ref{app:proof}.

\subsection{Striking-aware Evaluation Framework}
Given a query $(s_q, r_q, ?, t_q)$, a TKGR could output a score vector $\mathbf{y} \in \mathbb{R}^{|\mathcal{E}|}$. Through ranking $\mathbf{y}$, the rank of the answer entity could be obtained. The original evaluation method calculates the Mean Reciprocal Ranking (MRR) and Hits@k based on ranks. However, the approach assigns equal weight to all future events, making it unable to capture the model's ability to predict outstanding events. To address this limitation, we propose a striking-aware evaluation framework to evaluate existing TKGR baselines. Specifically, the computed strikingness scores are used as weighting factors to calculate the Weighted MRR (WMRR) and Weighted Hits@k (WHits@k) metrics, as described below:

\begin{flalign}
{\rm WMRR} = \frac{\sum_{i=1}^{|N|} (s_i+b)*\frac{1}{rank_i}}{\sum_{i=1}^{|N|} (s_i+b)}
\end{flalign}
\begin{flalign}
{\rm WHits@k} = \frac{\sum_{i=1}^{|N|}(s_i+b)*\mathbb{I}(rank_i \leq k )}{{\sum_{i=1}^{|N|} (s_i+b)}}
\end{flalign}
where $|N|$ is the size of the test set and $s_i$ is the strikingness of the event, which ensures that high-strikingness events contribute more to the metric. Setting $b$ is equivalent to assigning a higher cost of mis-prediction to outstanding events in the evaluation. Specifically, with $b=0.1$, an event with strikingness $sk=1$ receives approximately ten times the weight of an event with $sk=0$ in the metric calculation, since $(1+b)/(0+b)=11$. This reflects our value judgment that the utility of correctly predicting a critical outstanding event far outweighs that of correctly predicting a routine trivial one. The parameter $b$ thus quantifies and incorporates this value judgment into the evaluation framework.

In addition, we introduce a simple ensemble method. Instead of constructing complex networks, we follow \cite{meilicke2021naive,liu2023simfy,wang2024large} to combine the output scores of the existing path- and representation-based methods straightforwardly to obtain the final score:
\begin{flalign}
\label{ensmble_form}
\mathbf{y}_{ensem} = \eta \mathbf{y}_{path}+(1-\eta)\mathbf{y}_{representation}
\end{flalign}
where $\eta \in [0,1]$ is a hyperparameter. We perform a hyperparameter search for the $\eta$ on the validation set, and subsequently apply it to the test set. \textbf{Our purpose in constructing the ensemble method is not to pursue higher performance. Throughout the paper, we focus on investigating the boundaries of TKGR models' reasoning capabilities.}

\section{Experiments}
\subsection{Implement Settings}
\paragraph{Datasets} Extensive experiments are conducted on four TKG datasets: ICEWS14, ICEWS18, ICEWS05-15, and GDELT. The datasets are divided in chronological order. 

\begin{figure}[!thb]
    \centering
    \subfigure{
        \begin{minipage}[t]{0.95\linewidth}
            \centering
            \includegraphics[width=1.0\textwidth]{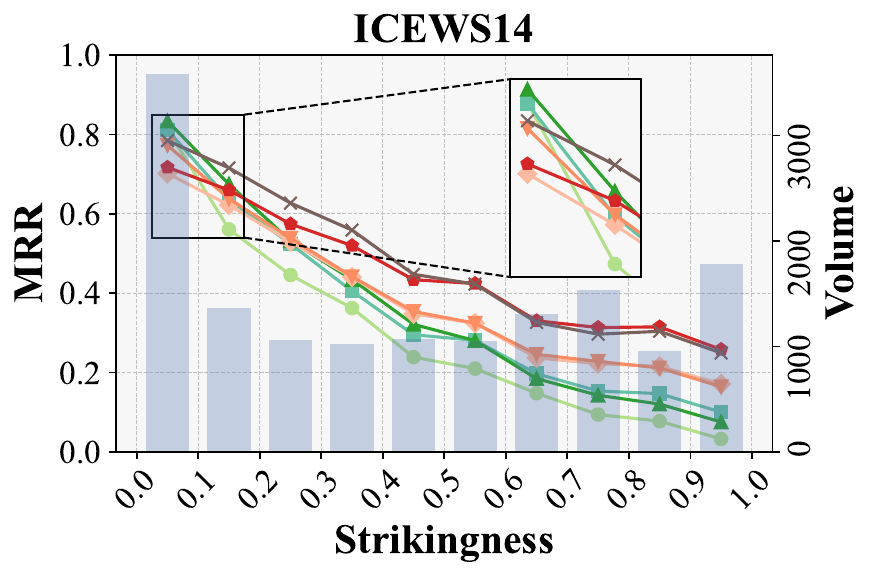}
        \end{minipage}
    }
    \subfigure{
        \begin{minipage}[t]{0.95\linewidth}
            \centering
            \includegraphics[width=1.0\textwidth]{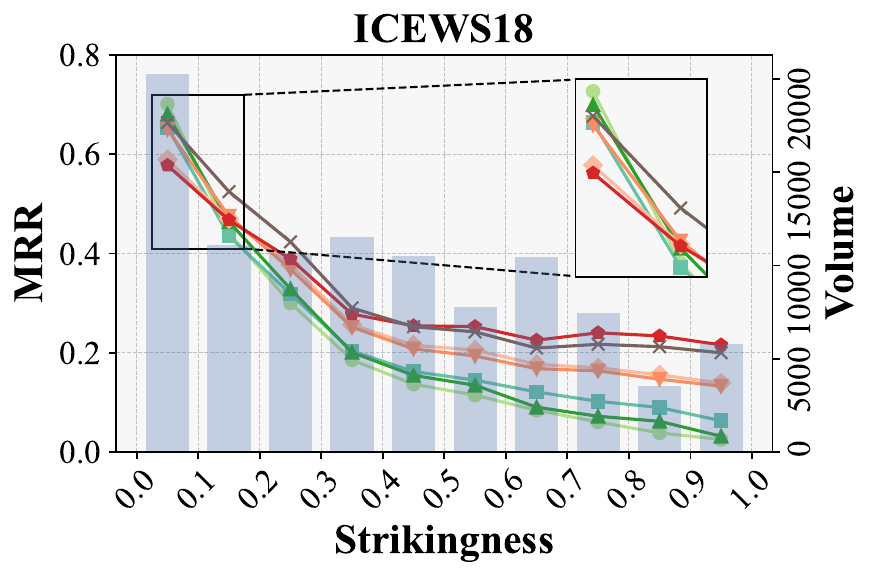}
        \end{minipage}
    }
    \subfigure{
        \begin{minipage}[t]{0.95\linewidth}
            \centering
            \includegraphics[width=1.0\textwidth]{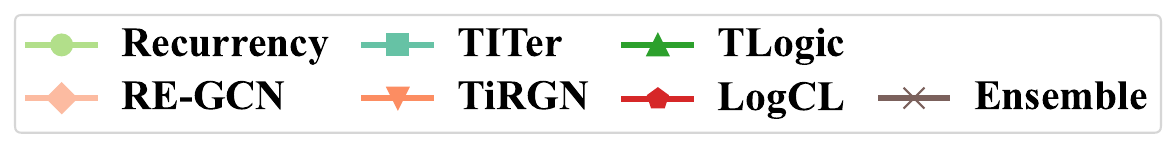}
        \end{minipage}
    }
    \caption{Group performances on ICEWS14 and ICEWS18. In each group, the bars denote the number of test events, while the lines indicate the average performance.}
    \label{Group_results}
\end{figure}

\paragraph{Baselines} We conducted comparisons under a unified experimental framework, focusing on reproducible approaches, including path-based methods: Recurrency \cite{Recurrency}, TITer \cite{TITer}, and TLogic \cite{liu2022tlogic}, representation-based methods: RE-GCN \cite{li2021temporal}, TiRGN \cite{TiRGN}, and LogCL \cite{LogCL}, and LLM-based methods: ICL \cite{ICL} and GenTKG \cite{liao2024gentkg}. We follow the time-aware filtering settings described in \cite{TKGFEvaluation}\footnote[1]{\href{https://github.com/nec-research/TKG-Forecasting-Evaluation}{https://github.com/nec-research/TKG-Forecasting-Evaluation}}. The code and strikingness weights are available\footnote[2]{\href{https://github.com/PersimmonZ1/RSMF}{https://github.com/PersimmonZ1/RSMF}}.

\paragraph{Hyperparameters} We set $\lambda$ as 0.1 following TLogic~\cite{liu2022tlogic}. And $\alpha^s$, $\alpha^o$, $\alpha^r$ are set to 0.4, 0.4, 0.2. The recommended parameters setting is provided to establish a well-calibrated evaluation framework to facilitate reproducibility and promote community adoption. More implementation details are provided in the Appendix \ref{appendix: implement details}.

\subsection{TKGR Performance on Group Strikingness}
\textbf{Performance Across Groups} Distinct from previous works that only report overall metrics for model performance, we group the data based on strikingness and compute the average performance within each group. Figure \ref{Group_results} presents the grouped MRR results for six baseline models and our proposed ensemble model on the ICEWS14. It could be observed that the volume of events decreases as strikingness increases. That is, low-strikingness events dominate the test set, whereas high-strikingness events are scarce, which aligns with human intuition. On performance, all models exhibit a decline with increasing strikingness, indicating that events with higher strikingness are more difficult to predict.

In overall trends, we find that path-based methods excel at predicting events with low strikingness, while representation-based methods demonstrate superior performance on high-strikingness events. This indicates that different categories of methods have distinct advantages in forecasting future events with different levels of strikingness. Based on this finding, we attempt to investigate whether a simple ensemble method can combine the advantages of both methods. Results on ICEWS0515 and GDELT exhibit consistent conclusions. Due to space limitations, they are shown in Appendix \ref{appendix:group strikingness}.

For the ensemble method, we observe two phenomena: \textbf{1) Trade-off}: For events with $sk < 0.1$ or $sk > 0.5$, the performance of the ensemble method lies between those of the individual methods. \textbf{2) Enhancement}: For events with $ 0.1 < sk < 0.5$, the ensemble method outperforms both individual methods. Based on these observations, it can be inferred that some models that leverage the ensemble method mainly improve the performance for low-strikingness events, which are predominant. For instance, comparing RE-GCN and TiRGN, which incorporates additional repeated events information, TiRGN significantly outperforms RE-GCN on low-strikingness events, while achieving similar performance on high-strikingness events. 

\begin{table}[!h]
    \centering
    \resizebox{0.49\textwidth}{!}{
    \begin{tabular}{c|c|cccc}
    \toprule
    \multirow{2.5}[0]*{\textbf{Model}} & \multirow{2.5}[0]*{\textbf{Type}} &\multicolumn{4}{c}{\textbf{ICEWS14}} \\
    \cmidrule(lr){3-6} 
    ~&~ &$S(0.6,0.7)$ &  $S(0.7,0.8)$ &  $S(0.8,0.9)$& $S(0.9,1.0)$ \\
    \midrule 
    \multirow{2}*{Recurrency} &High $NO_f$  &19.10	&13.46	&12.05	&5.23	\\
    ~&Low $NO_f$ &12.28	&5.77	&5.36	&1.65	\\
    \midrule
    \multirow{2}*{TITer}&High $NO_f$ &27.61	&21.92	&22.89	&16.67	 \\
    ~&Low $NO_f$ &17.51	&11.41	&9.13	&7.44	\\
    \midrule
    \multirow{2}*{TLogic}&High $NO_f$ &28.96	&21.28	&19.08	&14.60	\\
    ~&Low $NO_f$ &15.57	&10.90	&8.73	&4.55	 \\
    \midrule 
    \multirow{2}*{RE-GCN} &High $NO_f$  &38.36	&31.03	&30.92	&24.57	\\
    ~&Low $NO_f$ &17.81	&19.23	&17.86	&14.98	 \\
    \midrule
    \multirow{2}*{TiRGN}&High $NO_f$ &35.97	&31.15	&29.12	&21.78	 \\
    ~&Low $NO_f$ &19.31	&20.26	&18.25	&14.88	 \\
    \midrule
    \multirow{2}*{LogCL}&High $NO_f$ &49.40	&42.56	&43.98	&33.45	 \\
    ~&Low $NO_f$ &27.10	&28.46	&29.56	&25.83	\\
    \midrule
    \multirow{2}*{Ensemble}&High $NO_f$ &49.70	&40.64	&42.77	&32.60	 \\
    ~&Low $NO_f$ &26.65	&25.90	&28.37	&24.69	\\
    \bottomrule
    \end{tabular}
    }
    \caption{The Hits@3 metric of High and Low $NO_f$ events within the high-strikingness groups on ICEWS14.}
    \label{table:NO}
\end{table}

\begin{table*}[!thb]
\centering
\renewcommand{\arraystretch}{0.8}
\resizebox{\textwidth}{!}{
\begin{tabular}{l|l|ccc|ccc|ccc|ccc}
\toprule
\multirow{2}{*}{\textbf{Dataset}}  & \multirow{2}{*}{\textbf{Model}} & \multicolumn{3}{c|}{\textbf{(W)MRR}}   & \multicolumn{3}{c|}{\textbf{(W)Hits@1}}   & \multicolumn{3}{c|}{\textbf{(W)Hits@3}} & \multicolumn{3}{c}{\textbf{(W)Hits@10}}     \\
\cmidrule(lr){3-14} 
&& ORG $\uparrow$ & SK $\uparrow$& $\Delta$ $\downarrow$& ORG $\uparrow$& SK $\uparrow$& $\Delta$ $\downarrow$& ORG $\uparrow$& SK $\uparrow$& $\Delta$ $\downarrow$& ORG $\uparrow$& SK $\uparrow$& $\Delta$ $\downarrow$\\ 
\midrule
\multirow{9}{*}{\textbf{ICEWS14}}
& ICL &-&-&-& 32.40 & 15.95 & 50.77 \% & 45.94 & 26.18 & 43.01 \% & 56.59 & 36.59 & 35.34 \% \\
& GenTKG &-&-&-& 37.04 & 18.28 & 50.65 \% & 48.43 & 28.07 & 42.04 \% & 53.62 & 33.96 & 36.67 \% \\
\cmidrule(lr){2-14} 
& Recurrency & 37.12 & 19.47 & 47.55 \% & 29.69 & 13.62 & 54.13 \% & 40.75 & 21.49 & 47.26 \% & 51.26 & 30.75 & 40.01 \% \\
& TITer & 41.87 & 25.46 & 39.19 \% & 32.97 & 17.55 & 46.77 \% & 46.45 & 28.39 & 38.88 \% & 58.31 & 40.72 & 30.17 \% \\
& TLogic & 42.52 & 24.92 & 41.39 \% & 33.19 & 16.68 & 49.74 \% & 47.63 & 28.53 & 40.10 \% & 60.27 & 41.41 & 31.29 \% \\
\cmidrule(lr){2-14} 
& RE-GCN & 42.43 & 29.92 & 29.48 \% & 31.90 & 19.76 & 38.06 \% & 47.59 & 33.86 & 28.85 \% & 62.74 & 49.90 & 20.47 \% \\
& TiRGN & 44.45 & 30.43 & 31.54 \% & 33.77 & 20.30 & 39.89 \% & 49.57 & 34.09 & 31.23 \% & 64.89 & 50.69 & 21.88 \% \\
& LogCL & \underline{48.84} & \underline{38.17} & \textbf{21.85 \%} & \underline{37.76} & \underline{26.88} & \textbf{28.81 \%} & \underline{54.60} & \underline{43.23} & \textbf{20.82 \%} & \underline{70.43} & \underline{60.78} & \textbf{13.70 \%} \\
\cmidrule(lr){2-14} 
& Ensemble & \textbf{51.35} & \textbf{38.67} & \underline{24.69} \% & \textbf{40.23} & \textbf{27.18} & \underline{32.44} \% & \textbf{57.60} & \textbf{44.07} & \underline{23.49} \% & \textbf{72.56} & \textbf{61.49} & \underline{15.26} \% \\
\midrule
\multirow{9}{*}{\textbf{ICEWS18}}
& ICL & -& -& - & 19.27 & 9.42 & 51.12 \% & 31.35 & 17.87 & 43.0 \% & 43.97 & 28.58 & 35.00 \% \\
& GenTKG & -& -& -& 21.36 & 11.04 & 48.31 \% & 33.51 & 20.35 & 39.27 \% & 40.03 & 26.68 & 33.35 \% \\
\cmidrule(lr){2-14} 
& Recurrency & 28.66 & 15.97 & 44.28 \% & 20.77 & 10.22 & 50.79 \% & 32.25 & 18.05 & 44.03 \% & 43.54 & 26.83 & 38.38 \% \\
& TITer & 29.65 & 18.60 & 37.27 \% & 21.58 & 12.10 & 43.93 \% & 33.06 & 20.46 & 38.11 \% & 44.98 & 31.32 & 30.37 \% \\
& TLogic & 29.59 & 17.30 & 41.53 \% & 20.42 & 10.21 & 50.00 \% & 33.60 & 19.61 & 41.64 \% & 48.06 & 32.05 & 33.31 \% \\
\cmidrule(lr){2-14} 
& RE-GCN & 32.78 & 23.96 & 26.91 \% & 22.54 & 14.87 & 34.03 \% & 36.91 & 26.83 & 27.31 \% & 52.74 & 42.03 & 20.31 \% \\
& TiRGN & 33.54 & 23.59 & 29.67 \% & 22.92 & 14.21 & 38.00 \% & 38.09 & 26.67 & 29.98 \% & 54.38 & 42.25 & 22.31 \% \\
& LogCL & \underline{35.43} & \underline{28.35} & \textbf{19.98 \%} & \underline{24.09} & \textbf{17.79} & \textbf{26.15 \%} & \underline{40.22} & \underline{32.11} & \textbf{20.16 \%} & \underline{58.04} & \underline{49.66} & \textbf{14.44 \%} \\
\cmidrule(lr){2-14} 
& Ensemble & \textbf{37.74} & \textbf{28.49} & \underline{24.51} \% & \textbf{26.19} & \underline{17.77} & \underline{32.15} \% & \textbf{42.84} & \textbf{32.31} & \underline{24.58} \% & \textbf{60.70} & \textbf{50.15} & \underline{17.38} \% \\
\midrule
\multirow{7}{*}{\textbf{ICEWS0515}}
& Recurrency & 44.39 & 26.66 & 39.94 \% & 35.68 & 18.89 & 47.06 \% & 49.26 & 30.05 & 39.00 \% & 60.54 & 41.61 & 31.27 \% \\
& TITer & 48.03 & 31.39 & 34.65 \% & 38.61 & 22.43 & 41.91 \% & 53.03 & 34.85 & 34.28 \% & 65.42 & 48.81 & 25.39 \% \\
& TLogic & 46.56 & 30.65 & 34.17 \% & 35.48 & 20.50 & 42.22 \% & 53.27 & 35.65 & 33.08 \% & 67.25 & 50.69 & 24.62 \% \\
\cmidrule(lr){2-14} 
& RE-GCN & 47.93 & 33.73 & 29.63 \% & 37.32 & 23.39 & 37.33 \% & 53.78 & 38.33 & 28.73 \% & 68.16 & 54.23 & 20.44 \% \\
& TiRGN & 49.90 & 35.01 & 29.84 \% & 38.95 & 24.12 & 38.07 \% & 56.13 & 40.08 & 28.59 \% & 70.69 & 56.52 & 20.05 \% \\
& LogCL & \underline{56.95} & \underline{45.39} & \textbf{20.30 \%} & \underline{45.88} & \underline{33.62} & \textbf{26.72 \%} & \underline{63.73} & \underline{51.75} & \textbf{18.80 \%} & \underline{77.79} & \underline{68.38} & \textbf{12.10 \%} \\
\cmidrule(lr){2-14} 
& Ensemble & \textbf{58.53} & \textbf{46.10} & \underline{21.24} \% & \textbf{47.48} & \textbf{34.17} & \underline{28.03} \% & \textbf{65.43} & \textbf{52.61} & \underline{19.59} \% & \textbf{79.23} & \textbf{69.41} & \underline{12.39} \% \\
\midrule
\multirow{7}{*}{\textbf{GDELT}} 
& Recurrency & \underline{24.37} & 16.53 & 32.17 \% & \textbf{16.43} & 9.96 & 39.38 \% & \underline{26.79} & 17.98 & 32.89 \% & 39.70 & 29.10 & 26.70 \% \\
& TITer & 20.17 & 13.16 & 34.75 \% & 14.23 & 8.18 & 42.52 \% & 21.98 & 14.06 & 36.03 \% & 30.67 & 22.07 & 28.04 \% \\
& TLogic & 19.77 & 12.44 & 37.08 \% & 12.23 & 6.53 & 46.61 \% & 21.67 & 13.43 & 38.02 \% & 35.62 & 25.01 & 29.79 \% \\
\cmidrule(lr){2-14} 
& RE-GCN & 19.73 & 14.39 & 27.07 \% & 12.50 & 7.85 & 37.20 \% & 20.96 & 15.21 & 27.43 \% & 33.89 & 27.10 & 20.04 \% \\
& TiRGN & 21.25 & 15.41 & 27.48 \% & 13.27 & 8.38 & 36.85 \% & 22.81 & 16.38 & 28.19 \% & 37.01 & 29.20 & 21.10 \% \\
& LogCL & 23.74 & \underline{19.47} & \textbf{17.99 \%} & 14.62 & \textbf{10.79} & \textbf{26.20 \%} & 25.57 & \underline{20.88} & \textbf{18.34 \%} & \underline{42.33} & \underline{37.14} & \textbf{12.26 \%} \\
\cmidrule(lr){2-14} 
& Ensemble & \textbf{25.26} & \textbf{19.69} & \underline{22.05} \% & \underline{15.60} & \underline{10.71} & \underline{31.35} \% & \textbf{27.58} & \textbf{21.37} & \underline{22.52} \% & \textbf{45.03} & \textbf{38.00} & \underline{15.61} \% \\
\bottomrule
\end{tabular}
 }
\caption{Performance comparison of original (`ORG') and striking-aware (`SK') evaluation, the higher value means better performance ($\uparrow$). $\Delta$ represents the relative performance decrease across two evaluation settings, and the smaller value indicates that the model is less affected by repetitive bias ($\downarrow$). The best results are bolded, and the second-best results are underlined.}
\label{main performance}
\end{table*}

\begin{figure}[t]
    \centering
    \includegraphics[width=0.95\linewidth]{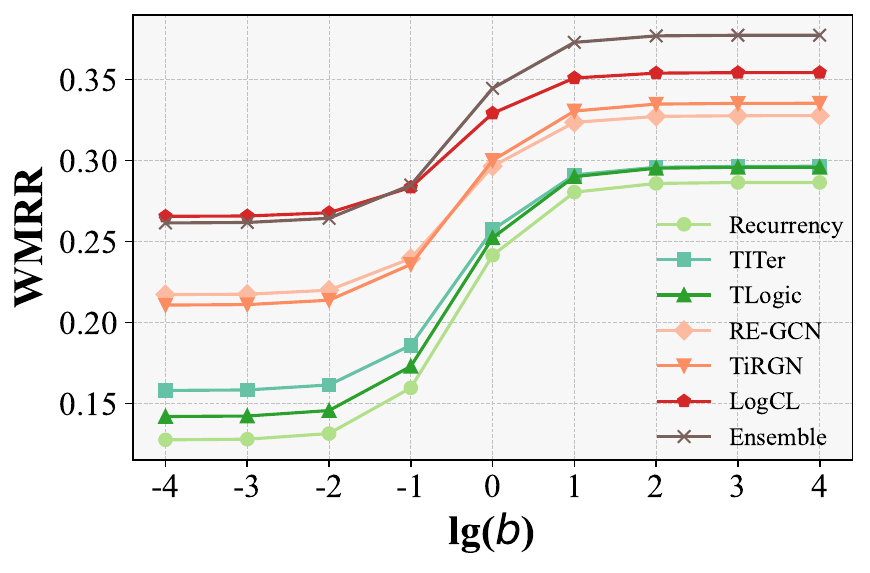}
    \caption{WMRR with different bias $b$ on ICEWS18.}
    \label{select_b}
\end{figure}
\paragraph{Predictability} To better understand the predictability of high-strikingness events, we introduce Neighborhood Overlap ($NO_f$), a structural metric that quantifies the richness of historical interactions between the subject and object entities. As shown in Table \ref{table:NO}, within each high-strikingness interval, events with higher $NO_f$ (richer historical evidence) are consistently more predictable than those with lower $NO_f$. This confirms that even among outstanding events, those supported by sufficient historical evidence remain learnable. The findings validate that our strikingness measure aligns not only with event rarity but also with predictive difficulty rooted in evidence scarcity. The definition of $NO_f$ and more analysis are provided in the Appendix \ref{appendix:Study of Event Predictability}.

\subsection{Strikingness-Aware Evaluation for TKGR}
To mitigate the low-strikingness bias in the dataset and more comprehensively evaluate the ability of TKGR models to forecast future events, we propose a striking-aware evaluation framework. We evaluate existing models using weighted MRR and weighted Hits@k. The hyperparameter $b$ determines the extent to which the evaluation results emphasize the model's ability to predict outstanding events. By adjusting $b$, the framework can place more or less weight on events, thereby controlling the balance performance between overall events and outstanding events. Figure \ref{select_b} shows the WMRR of the models on ICEWS18 under different $b$. A small $b$ indicates that WMRR focuses on the model's ability to predict outstanding events. When the b becomes very large, i.e., $b \ge 100$, the results of WMRR are close to the original MRR. Additionally, the comparison of model performance reverses as the value of $b$ changes. When $b > 0.1$, Ensemble outperforms LogCL, and TiRGN outperforms RE-GCN. However, when $b < 0.1$, the results change such that LogCL surpasses Ensemble, and RE-GCN outperforms TiRGN. This demonstrates that LogCL has a superior ability to predict outstanding events compared to the Ensemble method, and a similar conclusion holds for RE-GCN and TiRGN. 

\textbf{We empirically select $\boldsymbol{b = 0.1}$ to provide a unified evaluation result and facilitate fair comparison. This setting balances the metric's emphasis between the model's ability to reason events with varying levels of strikingness.} Table \ref{main performance} shows the experimental results of the models under both the original and the striking-aware evaluation frameworks. Obviously, the absolute values of the striking-aware metrics are significantly lower than those of the original evaluation framework, which aligns more closely with the recognized challenges of the future event forecasting task. Nevertheless, it is important to emphasize that our motivation is not to lower the scores purposely. Instead, we aim to reveal the models’ comprehensive predictive capabilities through the differences between the original and striking-aware metrics. Furthermore, adjustments to other hyperparameters have a negligible impact on model ranking evaluations, demonstrating the robustness of the strikingness-aware evaluation framework. Further details are provided in the Appendix \ref{appendix: Analysis Sensitivity}.

As shown, path-based methods exhibit a reduction of over 30\% across all datasets, with the heuristic baseline Recurrency demonstrating a particularly striking decrease of 50\%. In contrast, methods based on evolutionary representations experience a notably smaller decline, with all reductions remaining below 30\%. Among them, LogCL achieves the most robust performance, with its reduction consistently maintained at around 20\%. This indicates that the existing path-based methods have not demonstrated the claimed multi-hop reasoning on TKGs. Methods based on LLMs exhibit similar behavior to path-based approaches because the context window size constrains their reasoning, typically limited to first-order neighborhood information as input.

For the ensemble method, it still achieves state-of-the-art performance under the striking-aware metrics. However, the improvement is limited compared to the original metrics. This is because the ensemble approach primarily enhances the model's predictive performance on events with low strikingness, whereas the striking-aware evaluation framework focuses on the model's ability to predict outstanding events. By restricting the improvements from low-strikingness events, the striking-aware evaluation allows researchers to more objectively assess progress in the field. The results of other ensemble combinations are reported in Appendix \ref{appendix:Ensemble Combinations}.

\subsection{Outstanding Event Mining}
Through quantifying the strikingness of future events and extracting the top ones, we can mine outstanding events. Events in ICEWS18 with different strikingness as cases are showed in Table \ref{case_overall}. It is obvious that events with high strikingness are more likely to capture people's attention and may have a significant impact on the future.

\begin{table}[t]
    \small
    \centering
    \scalebox{1.0}{
    \begin{tabular}{l|l}
        \toprule
        Events & $sk$ \\
        \midrule
        $\text{Commando (Kosovo)}\xrightarrow[\text{2018-10-04}]{\text{Occupy territory}}\text{Serbia}$ & 1.0 \\
        $\text{Taliban}\xrightarrow[\text{2018-9-28}]{\text{Threaten with military force}}\text{Military (Afghanistan)}$ & 0.819 \\
        $\text{Buhari}\xrightarrow[\text{2018-10-12}]{\text{Mobilize or increase armed forces }}\text{Nigeria}$ & 0.704 \\
        $\text{Malaysia}\xrightarrow[\text{2018-10-15}]{\text{Sign agreement}}\text{Hong Kong}$ & 0.682 \\
        $\text{Russia}\xrightarrow[\text{2018-9-29}]{\text{Engage in material cooperation}}\text{China}$ & 0.575 \\
        $\text{Moon Jae-in}\xrightarrow[\text{2018-10-09}]{\text{intent to negotiate}}\text{Italy}$ & 0.349 \\
        $\text{India}\xrightarrow[\text{2018-10-05}]{\text{Make statement}}\text{Russia}$ & 0.103 \\
        $\text{France}\xrightarrow[\text{2018-10-27}]{\text{Consult}}\text{Germany}$ & 0.005 \\ 
        \bottomrule 
    \end{tabular}
    }
     \caption{Case events with different strikingness. }
    \label{case_overall}
\end{table}

To validate the effectiveness of the proposed RSMF in measuring event strikingness, we conducted a human evaluation study involving six volunteers. The six volunteers are all graduate students (Master's or Ph.D. candidates) in artificial intelligence, with research backgrounds spanning temporal knowledge graphs, knowledge graphs, relation extraction, and named entity recognition. Specifically, we randomly sampled 3000 events and a peer event for each target event, providing the contextual information for all events. The volunteers were asked to evaluate which event in each pair exhibited higher strikingness based on the given context.
\begin{table}[h]
    \small
    \centering
    \begin{tabular}{cccccc|c}
        \toprule
        H1 &H2  &H3 &H4 & H5& H6 & Average     \\
        \midrule
        0.683 &0.726  &0.698 &0.667 & 0.703& 0.696 & 0.696 \\
        \bottomrule 
    \end{tabular}
    \caption{Cohen's Kappa between humans and RSMF. }
    \label{human_evalutation}
\end{table}

Table \ref{human_evalutation} reports the Cohen's Kappa coefficients between the evaluations of individual annotators and the proposed RSMF. Cohen's Kappa is a widely used measure for inter-rater agreement, with values ranging from -1 to 1, where -1 denotes ``less than chance agreement" and 1 represents ``almost perfect agreement." As shown in Table \ref{human_evalutation}, the average Cohen's Kappa coefficient between RSMF and human evaluators is 0.696, indicating that RSMF achieves ``substantial agreement'' with human evaluators on the strikingness of events. 

We also conduct analysis experiments to validate four aspects of outstanding events: Novelty, Rarity, Context Dependence, and Time Sensitivity. The characteristics analysis is provided in Appendix \ref{appendix: Analysis Sensitivity}.

\section{Conclusion}
We observe that the current TKGR evaluation overweights trivial repetitive events, overshadowing models' ability to predict rare yet meaningful ones. To rectify this, we introduce a strikingness-aware evaluation framework that quantifies event strikingness through rule-based peer-event comparison  and incorporates it as a dynamic weight into ranking metrics.
Experiments on four benchmarks demonstrate: 1) a consistent performance drop as event strikingness increases, 2) a clear divide between path-based methods (strong on low-strikingness events) and representation-based methods (superior on high-strikingness events), and 3) we design a simple ensemble method and find it mainly improve prediction of repetitive events, with limited gains on rare, striking events.
Our framework recenters evaluation on the forecasting of outstanding events, offering a rigorous and meaningful benchmark for TKGR, and calls for future work to prioritize reasoning beyond repetition.

\section*{Acknowledgements}
This work was supported in part by the National Natural Science Foundation of China under Grant No. 62276110 and in part by the fund of Joint Laboratory of HUST and Pingan Property \& Casualty Research (HPL). The authors would also like to thank the anonymous reviewers for their comments on improving the quality of this paper.

\bibliographystyle{named}
\bibliography{ijcai26,reference}

\clearpage
\appendix
\section{Proof of Strikingness Boundedness under L2 Normalization}
\label{app:proof}

\begin{lemma}[Boundedness of $sk_f^{be}$ and $sk_f$]
\label{lem:boundedness}
Given the L2-normalized score vector $\mathbf{v} = \frac{\mathbf{sc}^{be}_{f}}{\|\mathbf{sc}^{be}_{f}\|_2}$ with all $v_i \geq 0$ and $\|\mathbf{v}\|_2 = 1$, the element-wise strikingness $sk_f^{be}$ defined in Eq.~(8) satisfies
\[
0 \leq sk_f^{be} \leq 1.
\]
Consequently, the overall strikingness $sk_f$ defined in Eq.~(9) also lies in $[0,1]$:
\[
0 \leq sk_f \leq 1.
\]
\end{lemma}

\begin{proof}
We prove the two inequalities separately.

\paragraph{1. Non-negativity of $sk_f^{be}$.}
By definition,
\[
sk_f^{be} = \sum_i v_i (v_i - v_f) \cdot \mathbb{I}(v_i > v_f),
\]
where $v_f$ is the score of the target event. Since $v_i \geq 0$ and the indicator function $\mathbb{I}(v_i > v_f)$ ensures that only terms with $v_i > v_f$ are included, each term $v_i(v_i - v_f)$ is non-negative. Hence $sk_f^{be} \geq 0$.

\paragraph{2. Upper bound of $sk_f^{be}$.}
Let $S = \{ i \mid v_i > v_f \}$. Then
\[
sk_f^{be} = \sum_{i \in S} v_i (v_i - v_f).
\]
Because $v_f \geq 0$, we have $v_i - v_f \leq v_i$ for all $i \in S$, and therefore
\[
v_i (v_i - v_f) \leq v_i^2.
\]
Summing over $i \in S$ yields
\[
sk_f^{be} \leq \sum_{i \in S} v_i^2.
\]
Since $\mathbf{v}$ is a unit vector in the L2 sense and all its components are non-negative,
\[
\sum_{i} v_i^2 = 1 \quad \text{and} \quad \sum_{i \in S} v_i^2 \leq 1.
\]
Thus $sk_f^{be} \leq 1$.

\paragraph{3. Boundedness of $sk_f$.}
Recall that
\[
sk_f = \alpha^s sk_f^s + \alpha^o sk_f^o + \alpha^r sk_f^r,
\]
where $\alpha^s, \alpha^o, \alpha^r \in [0,1]$ and $\alpha^s + \alpha^o + \alpha^r = 1$.
Since each $sk_f^{be} \in [0,1]$, we have
\[
0 \leq sk_f \leq \alpha^s \cdot 1 + \alpha^o \cdot 1 + \alpha^r \cdot 1 = 1.
\]
This completes the proof.
\end{proof}

\section{Experimental Setup}
\label{appendix: implement details}
\subsection{Details of Datasets}
The ICEWS datasets provide time-stamped political and socio-economic events curated from real-world interactions, while GDELT records a broad range of global events with fine-grained temporal annotations. For the TKGR task, the training, validation, and test sets are divided strictly in chronological order. It is worth mentioning that datasets like YAGO and WIKI are not used because they transform time-spanning facts into instantaneous ones, which does not align with our focus on event-centric temporal knowledge graphs.
\begin{table}[h]
    \caption{Details of the TKG datasets.}
    \setlength\tabcolsep{3pt}
    \renewcommand{\arraystretch}{1.1}
    \centering
    \scalebox{0.95}{
        \begin{tabular}{lllll}
            \toprule
            Dataset &ICEWS14  &ICEWS18 &ICEWS05-15 & GDELT     \\
            \midrule
            Entities &6,869  &23,033 &10,094 & 7,691  \\
            Relations & 230 & 256& 251 & 240\\
            Train & 74,845 & 373,018& 368,868& 1,734,399\\
            Valid & 8,514  &45,995 &46,302 &238,765\\
            Test & 7,371 &49,545 &46,159 &305,241\\
            Granularity & 24 hours &24 hours &24 hours & 15 mins\\
            \bottomrule
        \end{tabular}
        \label{datasets}
    }
\end{table}

\subsection{Implementation Details}
Our models were implemented using Python 3.10 and the PyTorch 1.13.1 framework. All experiments were conducted on a server equipped with an NVIDIA 3090TI GPU with 24GB memory, an Intel(R) i9-12900K CPU, and 256GB of RAM. The software environment includes CUDA 11.6 and cuDNN 8.4. For the ensemble method, we perform a grid search for the hyperparameter $\eta$ over the range $[0, 1]$ with a step size of $0.1$, and determine its optimal value by evaluating performance on the validation set of each dataset. We employ the rule mining framework from TLogic \cite{liu2022tlogic} to learn and extract temporal rules, with the distinction that we focus exclusively on rules of length 1. We confine the learning of rules, their confidence scores, and the calculation of strikingness strictly to the training set before each query timestamp, thereby preventing any test data leakage.

For a target event, $sc_f$ is defined as 0 when its set of peer events is empty. This may appear to assign high strikingness to a target event, which can be justified as follows: 1) If other peer events exist, a high strikingness score is reasonable, as the target event indeed fits the description of being unexpected. 2) If the peer event set for the target is empty, then according to Equation \ref{body_strikingness}, its $sk^{be}$ is calculated as 0. This ensures that events truly lacking evidential support are not incorrectly identified as high-strikingness. Moreover, in the ICEWS14 test set, fewer than 15 events (less than 0.2\%) have completely empty peer sets (subject, object, and relation). We argue that this negligible proportion of special cases does not affect the validity of our evaluation conclusions. 

In addition, if no historical rule supports any candidate replacement for the body element (i.e., \(\mathcal{C}^{be}_f = \emptyset\)), we define $\mathbf{sc}^{be}_{norm} = \mathbf{0}$. Consequently, if the normalized vector is zero, then \(sk^{be}_f = 0\) by Equation \ref{body_strikingness}, and the overall strikingness \(sk_f\) remains bounded in \([0,1]\).

\subsection{Original Metrics}
Given a test event $f=(s, r, o, t)$, its rank of score is computed by corrupting candidate entities. Specifically, the object entity would be replaced by another candidate $e_c$, and the candidate event $f_c = (s, r, c, t)$ would be scored. By adding the reverse events, i.e., $f'=(o, r^{-1}, s, t)$ and replacing $s$ by $c$, the subject entity prediction could be achieved. The two metrics are widely used to evaluate model performance:

\noindent{\textbf{Mean Reciprocal Rank (MRR)}} evaluates ranking quality by averaging the inverse rank of the correct result: 
\begin{flalign}
MRR = \frac{1}{2*|\mathcal{F}_{test}|}\sum_{f \in \mathcal{F}_{test}} (\frac{1}{rank_{f}}+\frac{1}{rank_{f'}}) 
\nonumber
\end{flalign}
\noindent{\textbf{Hits@\textit{k} (H@\textit{k})}}  measures the proportion of queries where the correct result appears in the top-k positions:
\begin{flalign}
Hits@k = \frac{1}{2*|\mathcal{F}_{test}|}\sum_{f \in \mathcal{F}_{test}}\mathbb{I}\{rank_{f} \leq k \} + \mathbb{I}\{rank_{f'} \leq k\}
\nonumber
\end{flalign}
where $\mathbb{I}\{True\}=1$ and $\mathbb{I}\{False\}=0$.

\subsection{Baselines}
In this paper, we conducted experiments with three categories of baselines: path-based, representation-based, and LLM-based methods. Previously, inconsistent experimental setups and dataset usage led to unfair comparisons. Therefore, instead of indiscriminately evaluating recent methods, we conducted comparisons under a unified experimental framework, focusing on community-recognized and reproducible approaches\footnote[1]{\href{https://github.com/nec-research/TKG-Forecasting-Evaluation}{https://github.com/nec-research/TKG-Forecasting-Evaluation}}. Additionally, we evaluate the recent state-of-the-art methods, such as LogCL and LLM-based methods, to ensure our study reflects the latest advances in the field. As LLM-based methods directly output the top 100 candidate answers without assigning explicit scores to each candidate, only Hits@k can be computed, whereas MRR is not applicable. Therefore, we did not report the performance of LLM-based methods in the grouped strikingness performance analysis. Moreover, due to the high computational cost of LLM inference, their application to large-scale datasets such as ICEWS05-15 and GDELT remains impractical in Table \ref{main performance}. The details of baselines are as follows:

\subsubsection{Path-based}
\paragraph{Recurrency} Recurrency assigns scores to candidates by calculating the recency and frequency of events related to the query, requiring only two hyperparameters to be searched in order to achieve baseline performance.
\paragraph{TLogic} TLogic generates answers by learning and applying rules to observed events before the query timestamp and scores the answer candidates relying on the rules’ confidences and time differences.
\paragraph{TITer} TITer treats the historical TKG as the environment and the historical domain as the action space. Starting from the subject, it employs reinforcement learning to navigate through the graph towards candidate entities.
\subsubsection{Representation-based}
\paragraph{RE-GCN} RE-GCN segments the historical TKG into a sequence of KG snapshots based on timestamps. It models the representations of entities and relations within each snapshot using Graph Convolutional Networks. Subsequently, the Recurrent Neural Network is employed to capture and compute the temporal evolution of these representations.
\paragraph{TiRGN} TiRGN extends RE-GCN by introducing a global historical encoder designed to gather repeated historical facts. 
\paragraph{LogCL} LogCL leverages contrastive learning between local and global representations to enhance the quality and robustness of the entity representations.
\subsubsection{LLM-based}
\paragraph{ICL} ICL directly leverages LLM's in-context learning ability. It formulates historical TKG events into sequential prompts without fine-tuning.

\paragraph{GenTKG} GenTKG is a retrieval-augmented generation framework that combines temporal logical rule-based retrieval (TLR) and few-shot parameter-efficient instruction tuning (FIT). It aligns LLMs with TKG forecasting, outperforming traditional methods with minimal training data.

\section{Computation Complexity Analysis}
\label{app:complexity}
In this section, we analyze the computational complexity of the Rule-based Strikingness Measuring Framework (RSMF). Let us define: $N_e$: number of entities, $N_r$: number of relations, $W$: historical window size (number of timestamps), $L$: length of temporal rules (number of events in rule body), $R$: number of mined rules, and $E_w$: average number of events in the historical window. For a target event $f=(s,r,o,t)$, the computation involves:

\paragraph{Peer Event Generation}
Generating peer events by replacing subject, object, or relation yields $O(N_e + N_r)$ candidates.

\paragraph{Rule Grounding}
For each peer event and each rule of length $L$, we search for matching rule bodies in the historical window. The number of possible matches for a length-$L$ rule is bounded by $O\left(E_w \cdot (N_e \cdot W)^{L-1}\right)$. Since the first event can be matched to any of the $E_w$ events in the window, and each subsequent event in the rule chain may involve a new entity and timestamp.

\begin{figure*}[!t]
    \centering
    \subfigure{
        \begin{minipage}[t]{0.3\linewidth}
            \centering
            \includegraphics[width=1.0\textwidth]{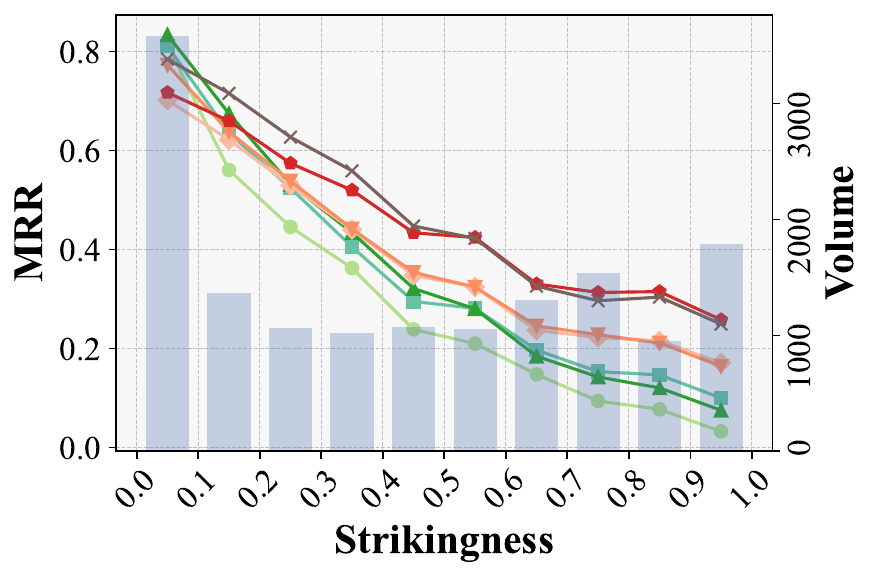}
        \end{minipage}
    }
    \subfigure{
        \begin{minipage}[t]{0.3\linewidth}
            \centering
            \includegraphics[width=1.0\textwidth]{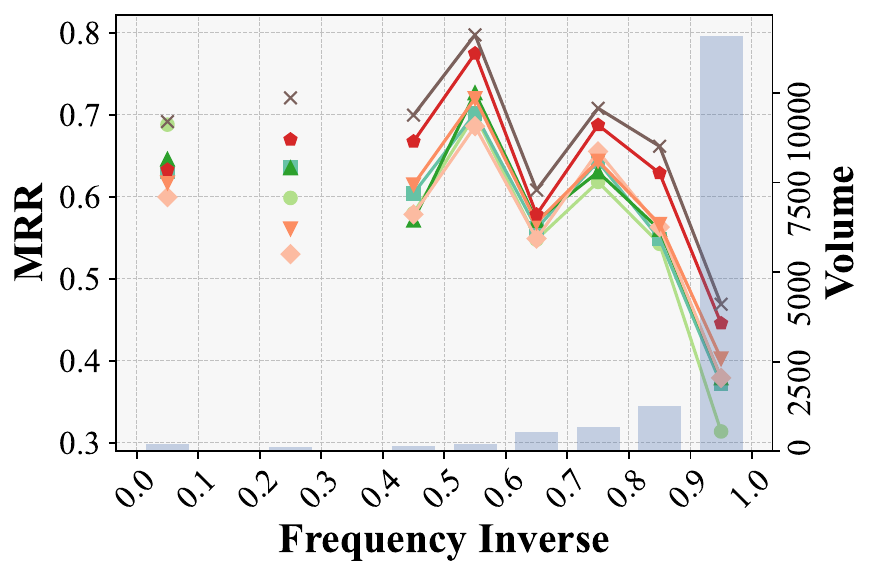}
        \end{minipage}
    }
    \subfigure{
        \begin{minipage}[t]{0.3\linewidth}
            \centering
            \includegraphics[width=1.0\textwidth]{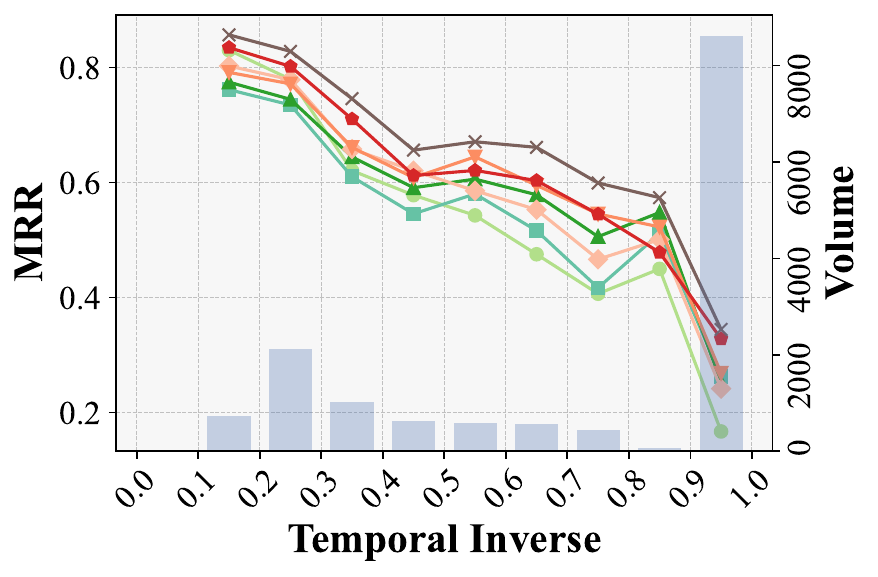}
        \end{minipage}
    }
    \subfigure{
        \begin{minipage}[t]{0.8\linewidth}
            \centering
            \includegraphics[width=1.0\textwidth]{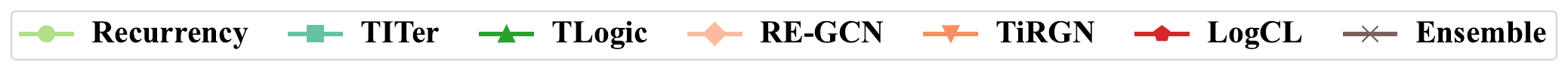}
        \end{minipage}
    }
    \caption{Comparing with other strikingness baselines.}
    \label{baseline_compare}
\end{figure*}

\begin{figure}[!thb]
    \centering
    \subfigure{
        \begin{minipage}[t]{1.0\linewidth}
            \centering
            \includegraphics[width=1.0\textwidth]{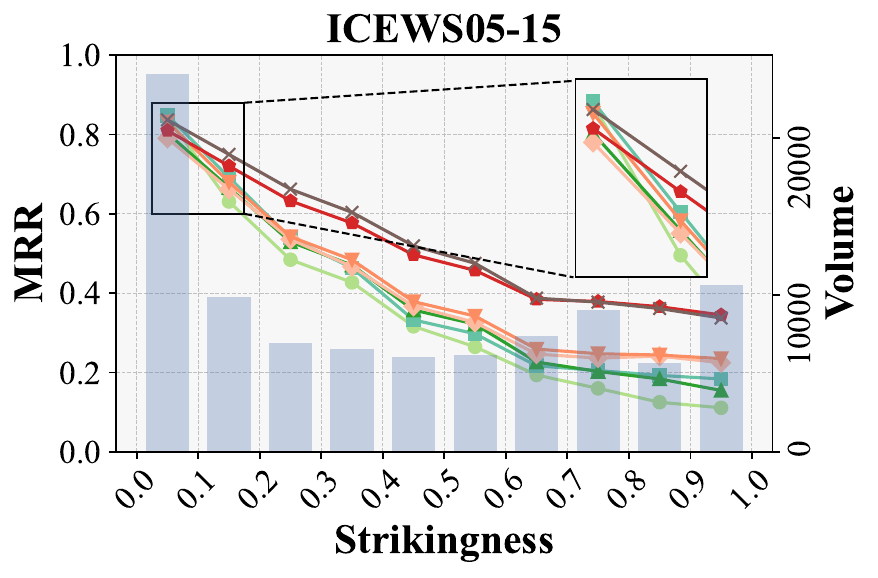}
        \end{minipage}
    }
    \subfigure{
        \begin{minipage}[t]{1.0\linewidth}
            \centering
            \includegraphics[width=1.0\textwidth]{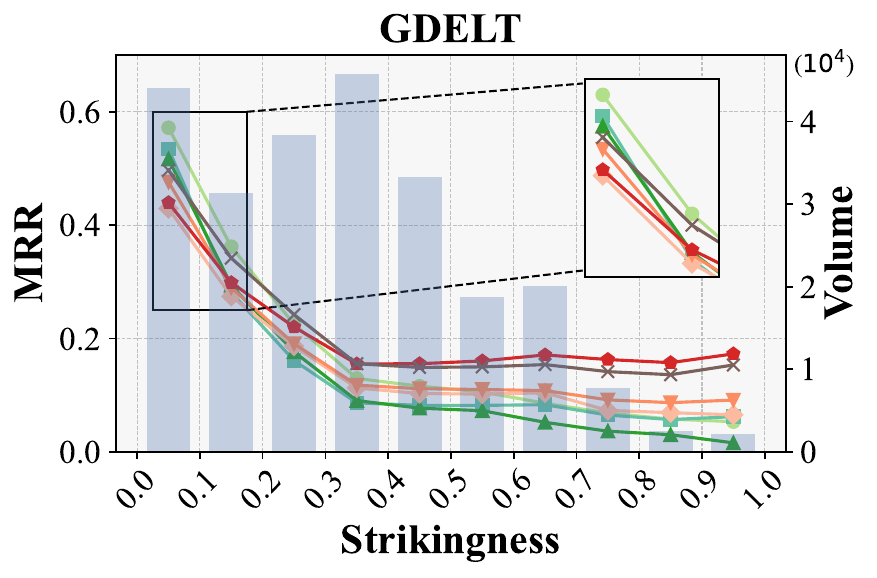}
        \end{minipage}
    }
    \subfigure{
        \begin{minipage}[t]{1.0\linewidth}
            \centering
            \includegraphics[width=1.0\textwidth]{Legend.pdf}
        \end{minipage}
    }
    \caption{Group performances of different models and data volume on ICEWS05-15 and GDELT.}
    \label{Group_results_additional}
\end{figure}

\paragraph{Total Complexity}
Considering all $R$ rules and $(N_e + N_r)$ peer candidates, the overall complexity per query is $\mathcal{O}\left( R \cdot (N_e + N_r) \cdot E_w \cdot (N_e \cdot W)^{L-1} \right)$. When $L=1$, the term $(N_e \cdot W)^{L-1}$ becomes $O(1)$, simplifying to $\mathcal{O}(R \cdot (N_e + N_r) \cdot E_w)$,
which is linear in the window size and entity count. However, for $L \geq 2$, the complexity grows exponentially with $L$, due to the factor $(N_e \cdot W)^{L-1}$.

\subsection{Why First-Order Rules Are Sufficient}
While we acknowledge that higher-order temporal rules ($L \geq$ 2) can capture more complex multi-hop dependencies, we deliberately restrict RSMF to first-order (length‑1) rules for the following reasons, which align with the primary goal of our work: to construct a practical and scalable strikingness-aware evaluation framework rather than to perform exhaustive temporal pattern mining:

\noindent{\textbf{Computational Feasibility}:} As shown above, longer rules lead to exponential growth in grounding complexity, making them infeasible for large-scale TKGs with thousands of entities and fine-grained timestamps (e.g., GDELT with 15-minute granularity).\\
\noindent{\textbf{Adequate Expressiveness for Strikingness}:} Strikingness is primarily concerned with whether an event is expected given its immediate historical patterns. First-order rules already capture the most direct temporal dependencies (e.g., ``if A visited B recently, A may visit B again''), which are sufficient for distinguishing repetitive vs. outstanding events.

We emphasize that our goal is not to claim that first‑order rules are universally optimal for all temporal pattern mining tasks. Rather, for the specific purpose of de-emphasizing trivial repetitions in TKGR evaluation. Future work may explore hybrid or approximate higher‑order strategies where computational resources permit, but such extensions are orthogonal to the core contribution of this paper.

\paragraph{Beyond Circular Reasoning} A legitimate concern is whether our strikingness measure creates a self-fulfilling evaluation: if high strikingness indicates lack of local evidence, methods relying on such evidence might appear to fail by definition. We argue this is not the case.

\begin{itemize}
\item \textbf{Path-based methods are assumed to perform multi-hop reasoning}, yet they perform worst on high-strikingness events. In fact, events that can only be supported by multi‑hop evidence, rather than one‑hop evidence, are more likely to be marked as high-strikingness events by RSMF. This reveals that their real strength lies in fitting shallow, repetitive patterns rather than multi-hop reasoning.
\item \textbf{Representation-based methods also decline} with increasing strikingness, though less severely, confirming that high strikingness corresponds to a general prediction challenge, not merely a lack of local evidence.
\item \textbf{Neighborhood overlap ($NO_f$) analysis} shows that even among high-strikingness events, those with richer historical interactions remain more predictable. This indicates RSMF captures difficulty beyond mere evidence absence.
\end{itemize}
\textbf{Using length-1 rules} ensures our measure is method-agnostic and focuses on immediate temporal expectations. The observed performance differences thus reflect true capability gaps, not circularity. Overall, RSMF does not penalize methods for lacking the evidence it uses; it exposes which methods can reason beyond trivial repetitions.

\section{More Group Results and Analysis}
\label{appendix:group strikingness}
Figure \ref{Group_results_additional} shows the grouped MRR results for baseline models and our proposed ensemble model on the ICEWS05-15 and GDELT datasets. Consistent with our analytical findings in the main manuscript, path-based methods yield better predictions for low-strikingness events, while representation-based approaches excel in high-strikingness scenarios. The ensemble method strikes a trade-off at both ends of the strikingness and demonstrates enhanced performance in the mid-range.
\subsection{Measuring Events with Frequency and Recency}
To contextualize RSMF, we compare it with two intuitive baseline measures: \textit{Frequency Inverse} (Freq Inv) and \textit{Temporal Inverse} (Temp Inv). Both are designed as simple proxies for event strikingness, yet our analysis reveals critical limitations that underscore the necessity of RSMF’s rule-based, peer‑comparative design:

\noindent\textbf{Frequency Inverse (Freq Inv)} measures historical uncommonness of the $(s,r)$ pair:
\begin{equation}
SK_{\text{freq}}(f) = 1 - \frac{\text{count}_{\mathcal{H}}(s,r)}{\max_{(s',r')}( \text{count}_{\mathcal{H}}(s',r'))}.
\end{equation}

\noindent\textbf{Temporal Inverse (Temp Inv)} considers recency of the exact event $(s,r,o)$:
\begin{equation}
Sk_{\text{temp}}(f) = 1 - \exp\left(-\lambda \cdot (t - t_{\text{last}})\right),
\end{equation}
where $t_{\text{last}}$ is its most recent occurrence time, and $\lambda=0.005$. If the event never occurred, $SK_{\text{time}}(f)=1$.

Figure \ref{baseline_compare} shows the group results of different strikingness measurements. 
\paragraph{Distribution Violates the Rarity Principle}
Visual inspection of the volume distributions reveals that both Freq Inv and Temp Inv assign a score ($1.0$) to a substantially larger proportion of test events. In RSMF’s grouping, the volume of events decays sharply as strikingness increases, producing the expected long‑tail distribution where truly outstanding events are rare. In contrast, Freq Inv and Temp Inv show a centralized volume distribution across strikingness bins, with a notably high volume remaining even in the highest bin ($[0.9, 1.0]$). This contradicts the basic premise that outstanding events should be scarce. The inflated high‑strikingness populations arise from inherent simplifications: Freq Inv relies on the long‑tail frequency of $(s,r)$ pairs, and Temp Inv treats any non‑recent exact repeat as striking, regardless of contextual expectation.

\paragraph{Predictive Correlation is Weak or Trivial.}
Model performance grouped by each measure reveals:
\begin{itemize}
\item \textbf{Freq Inv} exhibits erratic, non‑monotonic MRR trends across bins. The performance curve fluctuates without a clear gradient, indicating that frequency alone does not stably correlate with prediction difficulty.
\item \textbf{Temp Inv} shows a modest decline in MRR as strikingness increases, but this primarily reflects the trivial fact that events without recent repetitions are harder to predict.
\end{itemize}
Crucially, \textbf{neither baseline can discriminate between model families}. Under RSMF, the performance gap between path‑based and representation‑based methods widens considerably as strikingness increases. Under Freq Inv and Temp Inv, this gap remains narrow and inconsistent across bins, failing to expose the models’ distinct capabilities, especially for the ensemble method.

\paragraph{Why RSMF is Necessary.}
RSMF’s rule‑grounded peer‑event comparison incorporates both semantic confidence and temporal decay, enabling it to:
\begin{enumerate}
\item produces a realistic long‑tail strikingness distribution,
\item creates a sharp, monotonic gradient of prediction difficulty,
\item reveals systematic differences in model‑family performance, and
\item offer explainable strikingness assessments through rules and peer events.
\end{enumerate}

\subsection{Statistical Test Details}
We conducted two complementary statistical tests to compare model performance between the lowest strikingness bin ($sk<0.2$) and the highest bin ($sk>0.8$):

\noindent{\textbf{Welch’s t‑test}} (independent samples, unequal variances assumed)

\noindent{\textbf{Mann‑Whitney U test}} (non‑parametric, one‑sided alternative that low‑strikingness performance is greater)

Tests were performed separately for Hits@1 and Hits@3 metrics. Sample sizes are 31,078 for the low bin and 9,316 for the high bin across all models. All comparisons yield highly significant results ($p<0.001$). The t‑statistics are exceptionally large ($t>60$), and the U‑statistics are consistently on the order of $10^8$, reflecting both large effect sizes and substantial sample sizes. The consistency between parametric (t‑test) and non‑parametric (U‑test) results confirms the robustness of the conclusion. The extreme significance levels ($p\ll 0.001$) are expected given the large sample sizes ($n>40,000$ combined per test) and large observed differences ($\Delta$ Hits@3$>35\%$ for all models). These statistical tests provide formal confirmation that the performance degradation for high‑strikingness events is not due to chance variation.

\begin{table}[!thb]
    \centering
    \renewcommand{\arraystretch}{0.8}
    \resizebox{0.48\textwidth}{!}{
    \begin{tabular}{c|c|cccc}
    \toprule
    \multirow{2.5}[0]*{\textbf{Model}} & \multirow{2.5}[0]*{\textbf{Type}}  & \multicolumn{4}{c}{\textbf{ICEWS18}} \\
    \cmidrule(lr){3-6}
    ~&~ &$S(0.6,0.7)$ &  $S(0.7,0.8)$ &  $S(0.8,0.9)$& $S(0.9,1.0)$  \\
    \midrule 
    \multirow{2}*{Recurrency} &High $NO_f$  	&10.40	&8.94	&7.26	&4.13\\
    ~&Low $NO_f$ 	&8.05	&5.06	&2.97	&2.20 \\
    \midrule
    \multirow{2}*{TITer}&High $NO_f$ 	&15.71	&13.86	&14.32	&10.22 \\
    ~&Low $NO_f$ 	&10.19	&7.81	&5.11	&4.51 \\
    \midrule
    \multirow{2}*{TLogic}&High $NO_f$	&13.74	&11.32	&12.02	&6.24 \\
    ~&Low $NO_f$ 	&6.56	&4.71	&2.82	&2.05 \\
    \midrule 
    \multirow{2}*{RE-GCN} &High $NO_f$  	&23.55	&23.19	&22.36	&20.43\\
    ~&Low $NO_f$ &15.80	&14.88	&10.91	&12.42 \\
    \midrule
    \multirow{2}*{TiRGN}&High $NO_f$	&22.12	&22.98	&21.37	&18.89 \\
    ~&Low $NO_f$&15.65	&14.86	&10.80	&11.90 \\
    \midrule
    \multirow{2}*{LogCL}&High $NO_f$ &31.95	&32.52	&30.30	&28.17 \\
    ~&Low $NO_f$ &18.96	&21.77	&21.45	&22.71 \\
    \midrule
    \multirow{2}*{Ensemble}&High $NO_f$ &29.09	&29.16	&28.37	&25.95 \\
    ~&Low $NO_f$&17.79	&19.62	&20.25	&21.31 \\
    \bottomrule
    \end{tabular}
    }
    \caption{The Hits@3 metric of High and Low $NO_f$ events within the high-strikingness range on ICEWS18.}
    \label{table:NO2}
\end{table}

\begin{table}[t]
\setlength{\tabcolsep}{2pt}
\centering
\begin{tabular}{lccccc}
\toprule
\multirow{2}{*}{Model Type} &\multicolumn{2}{c}{ICEWS14} &\multicolumn{2}{c}{ICEWS18} \\
\cmidrule(r){2-3} \cmidrule(r){4-5} &$S(0,0.1)$ &$S(0.1,0.2)$ &$S(0,0.1)$ &$S(0.1,0.2)$ \\
\midrule
6-model-H@3  &62.70  &41.85 &45.81 &26.68 \\[0.2em]
5-model-H@3  &76.12  &59.09 &60.46 &40.07 \\[0.2em]
4-model-H@3  &84.85  &69.55 &71.09 &49.89 \\
\bottomrule
\end{tabular}
\caption{The Hits@3 of multi-models' intersection in the low-strikingness groups on ICEWS14 and ICEWS18.}
\label{table:trivial_event}
\end{table}

\section{More Predictability Analysis}
\label{appendix:Study of Event Predictability}
The $NO_f$ measures the degree of overlap between the neighbors of the subject entity and the object entity for the test sample $f=(s,r,o,t)$ in the historical KGs, which indicates the existence of more complex multi-hop historical interactions between the subject and object of event $f$. The formula is as follows:
\begin{flalign}
NO_f=\frac{||N_s\cap N_o ||}{||N_s\cup N_o||}
\end{flalign}
where $N_s=\{o'|(s,r',o',t')\} \cup \{o'|(o',r',s,t')\}$ and $N_o=\{s'|(s',r',o,t')\} \cup \{s'|(o,r',s',t')\}$ denote the set of neighboring entities of s and o, and $t-w\leq t' <t$ is consistent with the window for calculating strikingness. 

The results of predictability analysis on high-strikingness groups of ICEWS18 are shown in Table \ref{table:NO2}. Consistent with our prior findings, future events with higher $NO_f$ exhibit higher prediction accuracy.

\subsection{Predict Pattern on Low-strikingness Events across Different Models}  
We verify whether events in the \textbf{low-strikingness groups} exhibit the same easy patterns, that is, whether different models consistently predict mostly the same trivial events correctly. $S(sk_1, sk_2)$ represents the set of events with strikingness in the range $[sk_1, sk_2)$, and ``n‑model‑H@3" indicates that the Hits@3 predictions of n out of the six models overlap. We calculated the intersection of the Hits@3 metric of models within the low-strikingness groups. The results in Table \ref{table:trivial_event} demonstrate that for events with extremely low strikingness $S(0, 0.1)$, multiple models consistently make correct predictions. However, for events with strikingness in $S(0.1, 0.2)$, the overlap in model predictions drops sharply, indicating that even for trivial events, different models possess distinct prediction patterns. It further explains why the ensemble method in Figure \ref{Group_results} exhibits an enhancement pattern. 

Let $\mathcal{Q}$ be the set of test queries and $\mathcal{M} = {M_1, M_2, \dots, M_N}$ denote $N$ baseline models. For a query $q \in \mathcal{Q}$ and a model $M_i \in \mathcal{M}$, we define an indicator function $\mathbb{I}_{\text{Hits@3}}(M_i, q)$ to indicate whether model  $M_i$ achieves Hits@3 = 1 on query $q$. If yes, $\mathbb{I}_{\text{Hits@3}}(M_i, q)=1$, otherwise 0. If at least $n$ models in $\mathcal{M}$ simultaneously achieve Hits@3 = 1 on $q$, i.e., $\sum_{i=1}^{N} \mathbb{I}_{\text{Hits@3}}(M_i, q) \ge n$ .
We say the query $q$ satisfies the $n$-Model-H@3 condition. Consequently, the overall $n$-Model-Hits@3 performance on the test set $\mathcal{Q}$ is reported as follows:

\begin{equation}
\frac{\sum_{q\in\mathcal{Q}}\sum_{i=1}^{N} \mathbb{I}_{\text{Hits@3}}(M_i, q) \ge n}{|\mathcal{Q}|} .
\label{eq:condition}
\end{equation}

\begin{figure}[!thb]
    \centering
        \includegraphics[width=0.8\linewidth]{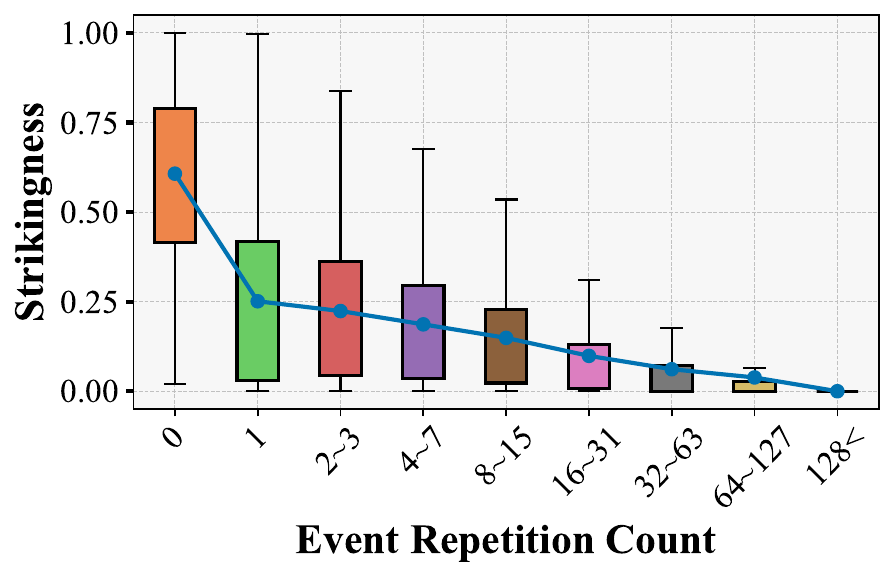}
    \caption{Relation between strikingness and novelty.}
    \label{Novelty Study}
\end{figure}

\begin{figure}[!t]
    \centering
    \includegraphics[width=0.8\linewidth]{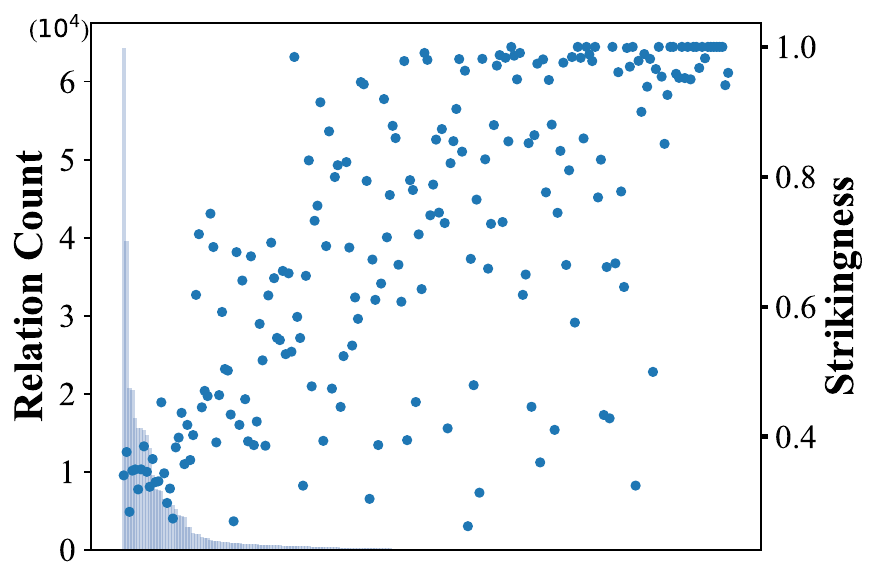}
    \caption{Strikingness and count of events with different relations on ICEWS18.}
    \label{Few-shot Study}
\end{figure}

\begin{figure*}[!t]
    \centering
    \subfigure[ICEWS14]{
        \begin{minipage}[t]{0.3\linewidth}
            \centering
            \includegraphics[width=1.0\textwidth]{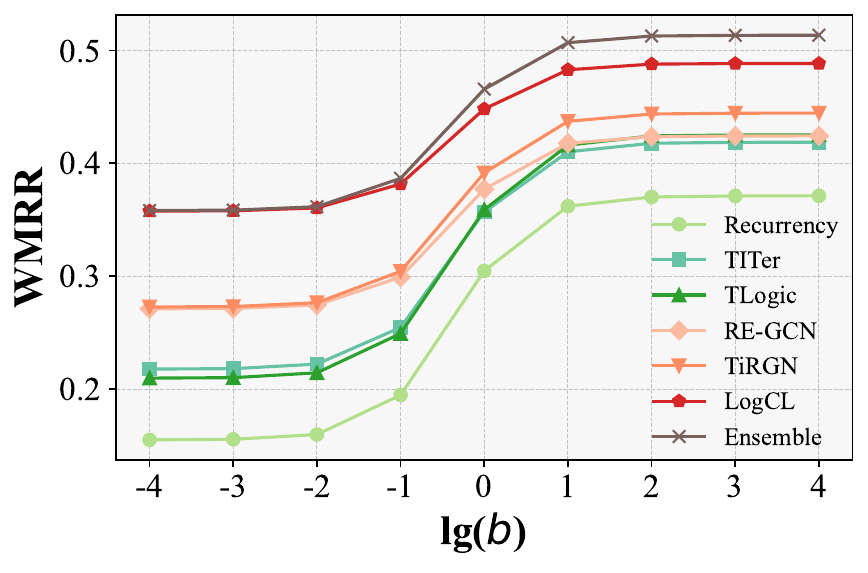}
        \end{minipage}
    }
    \subfigure[ICEWS05-15]{
        \begin{minipage}[t]{0.3\linewidth}
            \centering
            \includegraphics[width=1.0\textwidth]{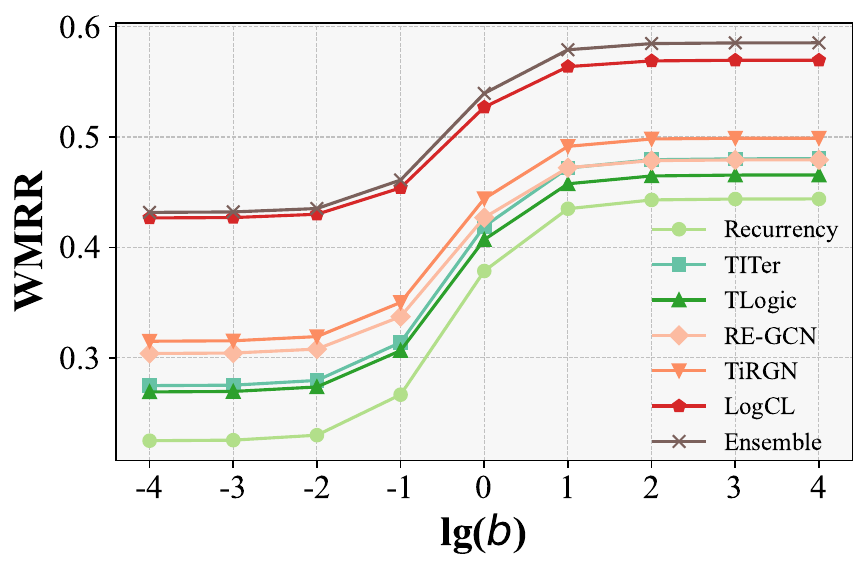}
        \end{minipage}
    }
    \subfigure[GDELT]{
        \begin{minipage}[t]{0.3\linewidth}
            \centering
            \includegraphics[width=1.0\textwidth]{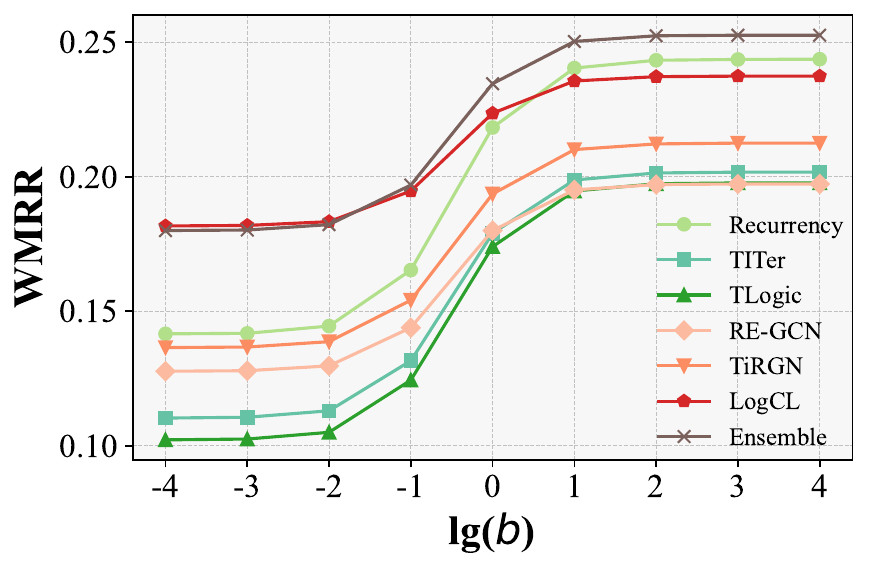}
        \end{minipage}
    }
    \subfigure{
        \begin{minipage}[t]{0.8\linewidth}
            \centering
            \includegraphics[width=1.0\textwidth]{Legend_1.pdf}
        \end{minipage}
    }
    \caption{WMRR with different bias $b$ on ICEWS14, ICEWS05-15, and GDELT.}
    \label{select_b_more}
\end{figure*}

\section{Analysis of Strikingness}
\label{appendix: Analysis Sensitivity}
\subsection{Characteristics of Strikingness}
We conduct analysis experiments to validate four aspects of outstanding events: Novelty, Rarity, Context Dependence, and Time Sensitivity.

\paragraph{Novelty of Outstanding Event} 
In Figure \ref{Novelty Study}, we explored the relationship between events' historical repetition count and strikingness on ICEWS14. The boxes represent the distribution of strikingness for events with a given repetition count, while the blue line illustrates the average strikingness. It can be observed that the strikingness of events decreases sharply as the historical repetition count increases, with the most pronounced drop occurring between first-occurring events and those that have occurred before. Nevertheless, the box plot reveals that even first-occurring events can demonstrate low strikingness. This is because, although an event may not have occurred historically, it can still be considered anticipated as long as there exists sufficient rules to support its occurrence. Conversely, highly frequent events are typically unlikely to be perceived as outstanding due to their routine nature.

\paragraph{Rarity of Outstanding Event}
Additionally, we further explored the relationship between strikingness and event relation types in Figure \ref{Few-shot Study}. As shown, events in TKGs, categorized by relation types, generally follow a long-tailed distribution. Moreover, the strikingness of relation-specific events tends to increase as the count of events associated with a relation decreases, suggesting that events with few-shot relation are more likely to be outstanding. However, it is also observed that some events with few-shot relation exhibit very low strikingness. It is because, although events involving few-shot relations constitute a relatively small proportion, they may occur repeatedly.

\begin{figure}[!t]
    \centering
    \subfigure{
        \begin{minipage}[t]{0.46\linewidth}
            \centering
            \includegraphics[width=1.05\linewidth]{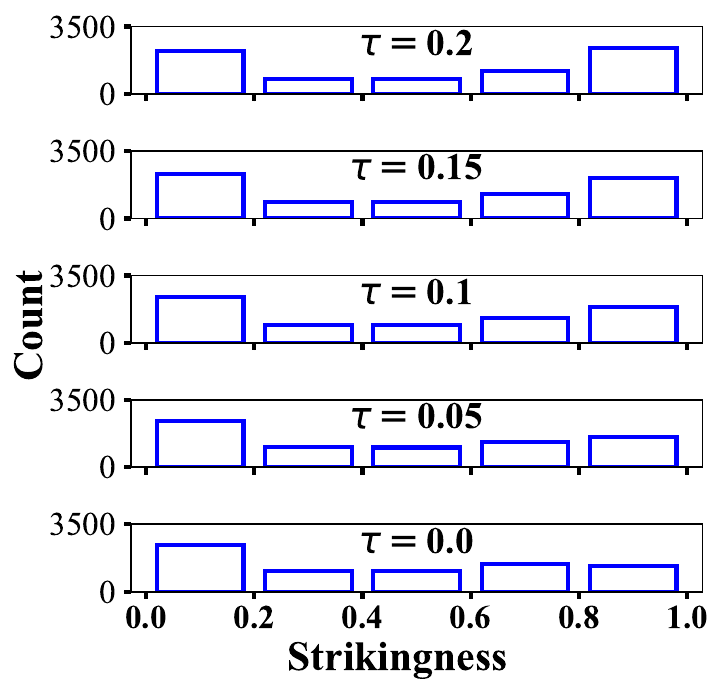}
        \end{minipage}
    }
    \subfigure{
        \begin{minipage}[t]{0.46\linewidth}
            \centering
            \includegraphics[width=1.05\linewidth]{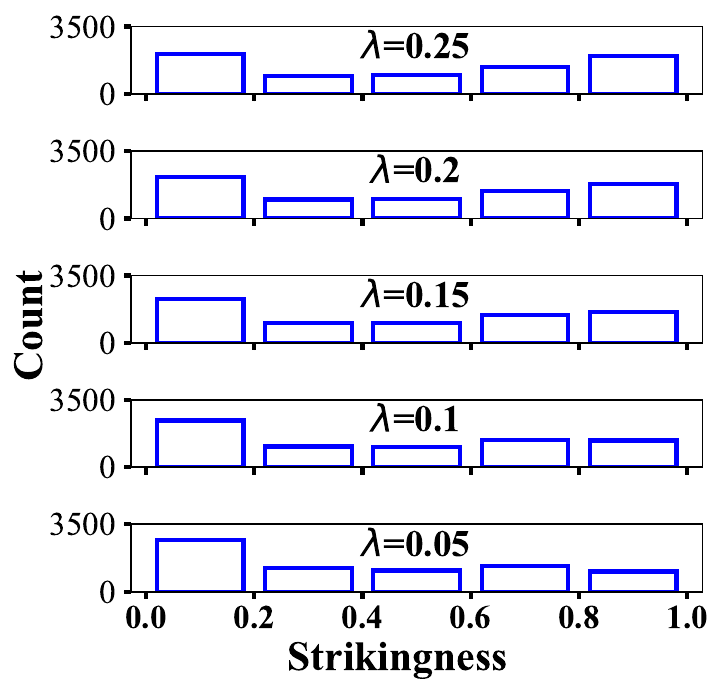}
        \end{minipage}
    }
    \subfigure{
        \begin{minipage}[t]{0.46\linewidth}
            \centering
            \includegraphics[width=1.05\linewidth]{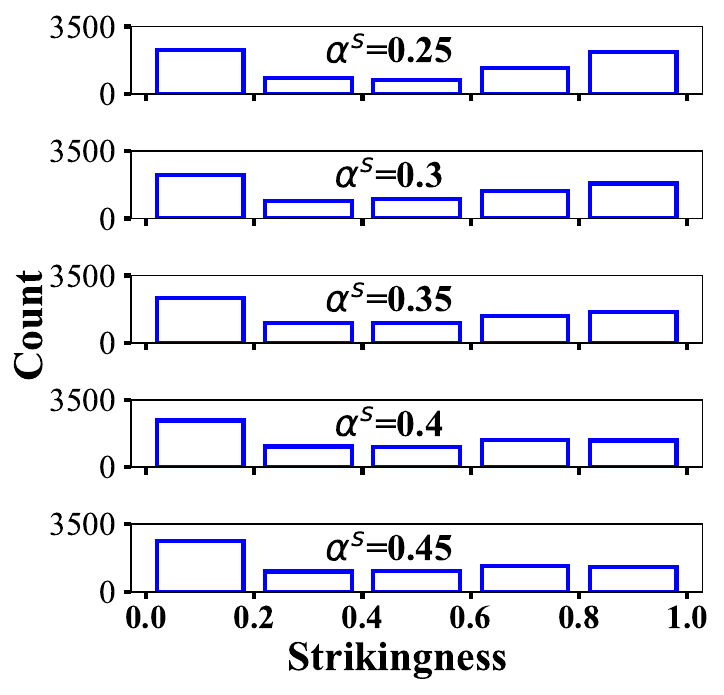}
        \end{minipage}
    }
    \subfigure{
        \begin{minipage}[t]{0.46\linewidth}
            \centering
            \includegraphics[width=1.05\linewidth]{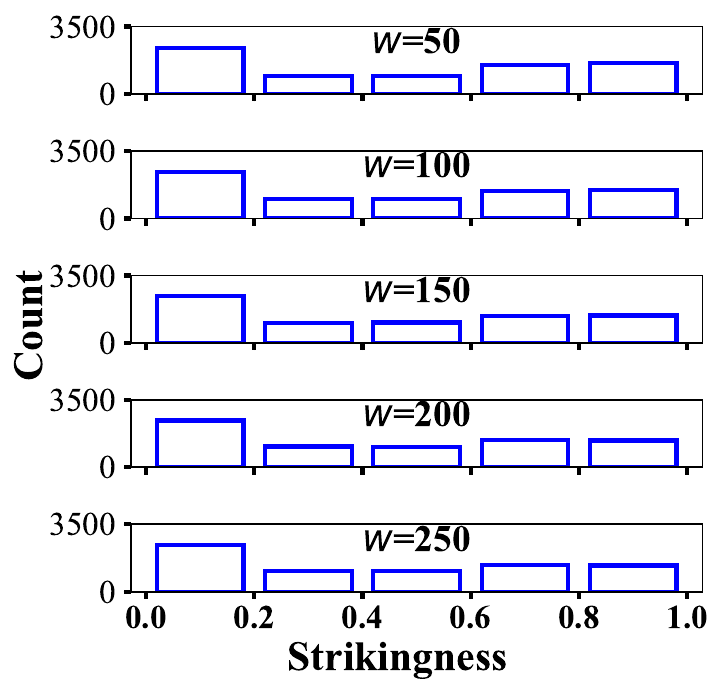}
        \end{minipage}
    }
    \caption{Distribution on ICEWS14 with different hyperparameters.}
    \label{Dist Study}
\end{figure}
\begin{figure}[!t]
    \centering
    \subfigure{
        \begin{minipage}[t]{0.46\linewidth}
            \centering
            \includegraphics[width=1.05\linewidth]{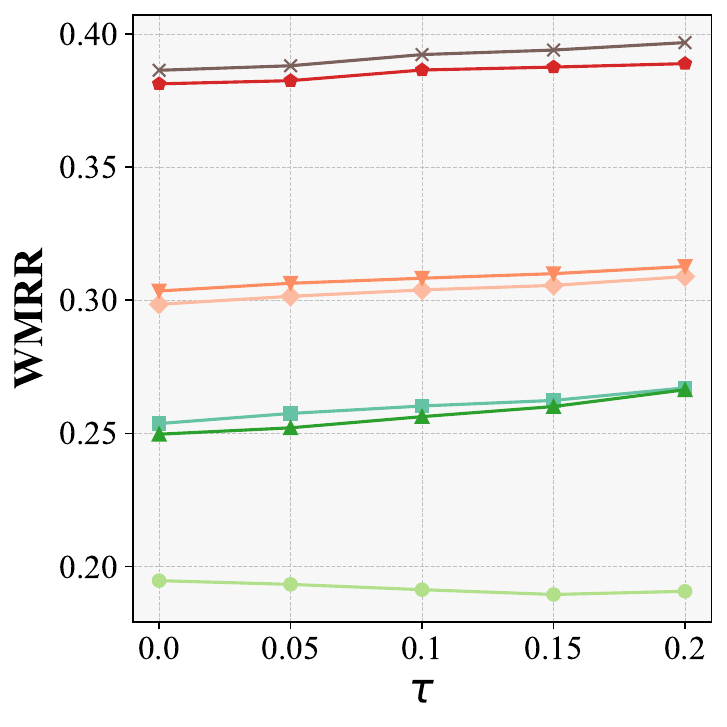}
        \end{minipage}
    }
    \subfigure{
        \begin{minipage}[t]{0.46\linewidth}
            \centering
            \includegraphics[width=1.05\linewidth]{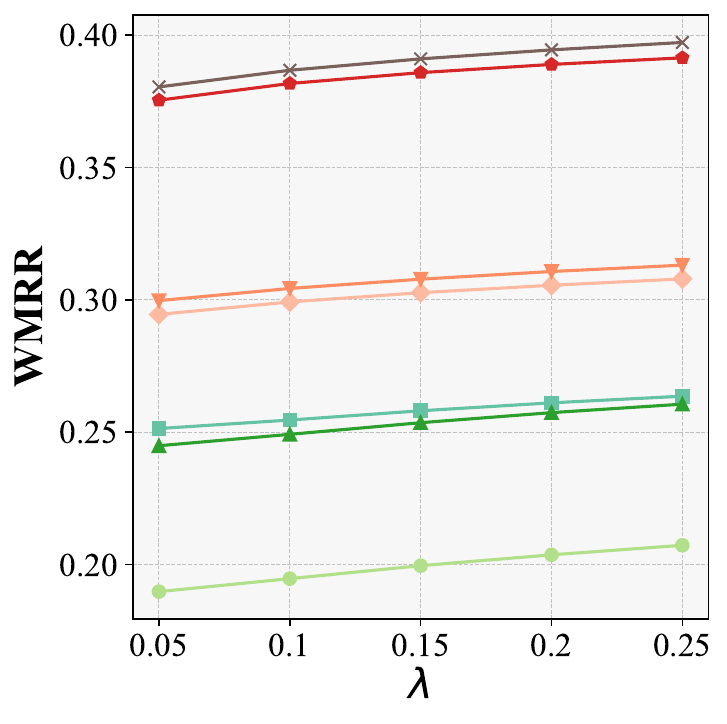}
        \end{minipage}
    } 
    \subfigure{
        \begin{minipage}[t]{0.46\linewidth}
            \centering
            \includegraphics[width=1.05\linewidth]{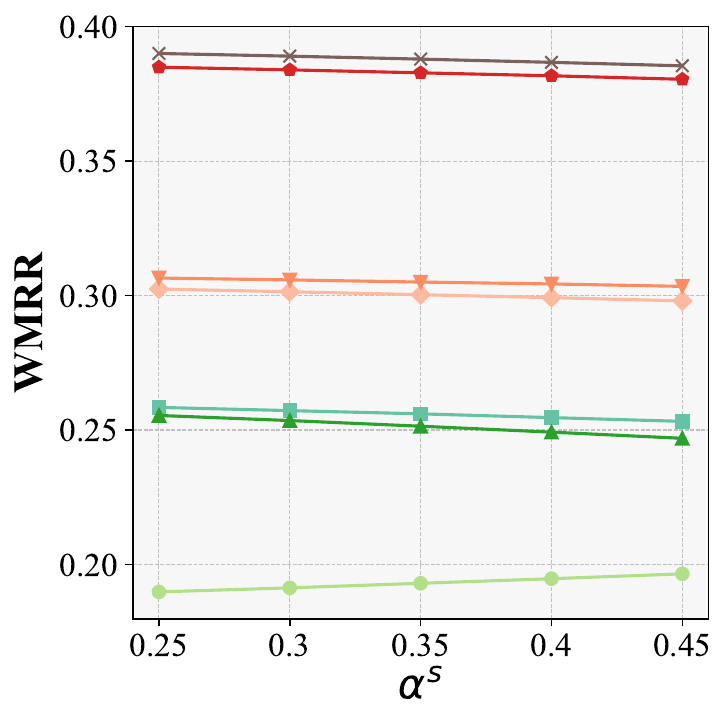}
        \end{minipage}
    }
    \subfigure{
        \begin{minipage}[t]{0.46\linewidth}
            \centering
            \includegraphics[width=1.05\linewidth]{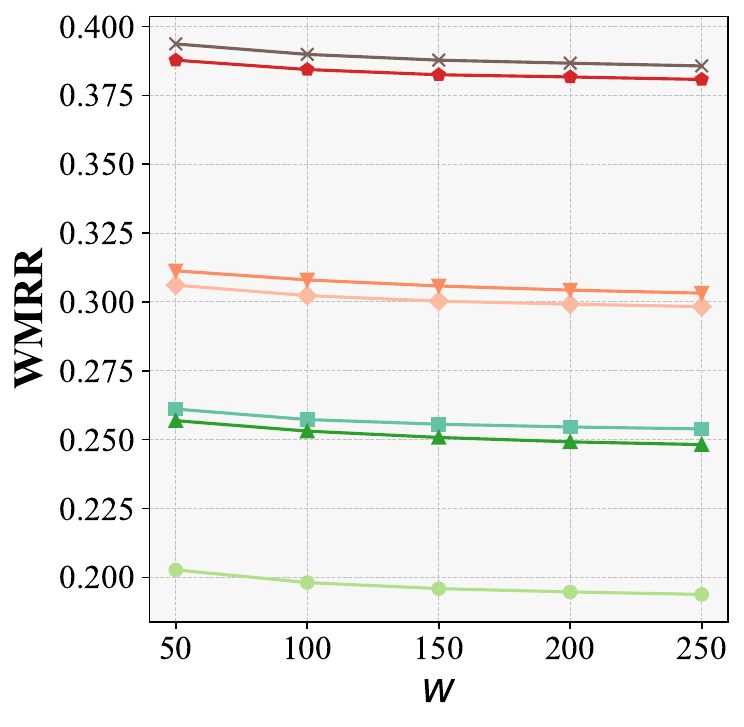}
        \end{minipage}
    }
    \subfigure{
        \begin{minipage}[t]{0.85\linewidth}
            \centering
            \includegraphics[width=1.0\textwidth]{Legend.pdf}
        \end{minipage}
    }
    \caption{Performance on ICEWS14 under strikingness-aware evaluation framework with $b=0.1$ and different hyperparameters.}
    \label{Hyper Study}
\end{figure}

\begin{figure*}[!t]
    \centering
    \subfigure[ICEWS14]{
        \begin{minipage}[t]{0.4\linewidth}
            \centering
            \includegraphics[width=1.05\linewidth]{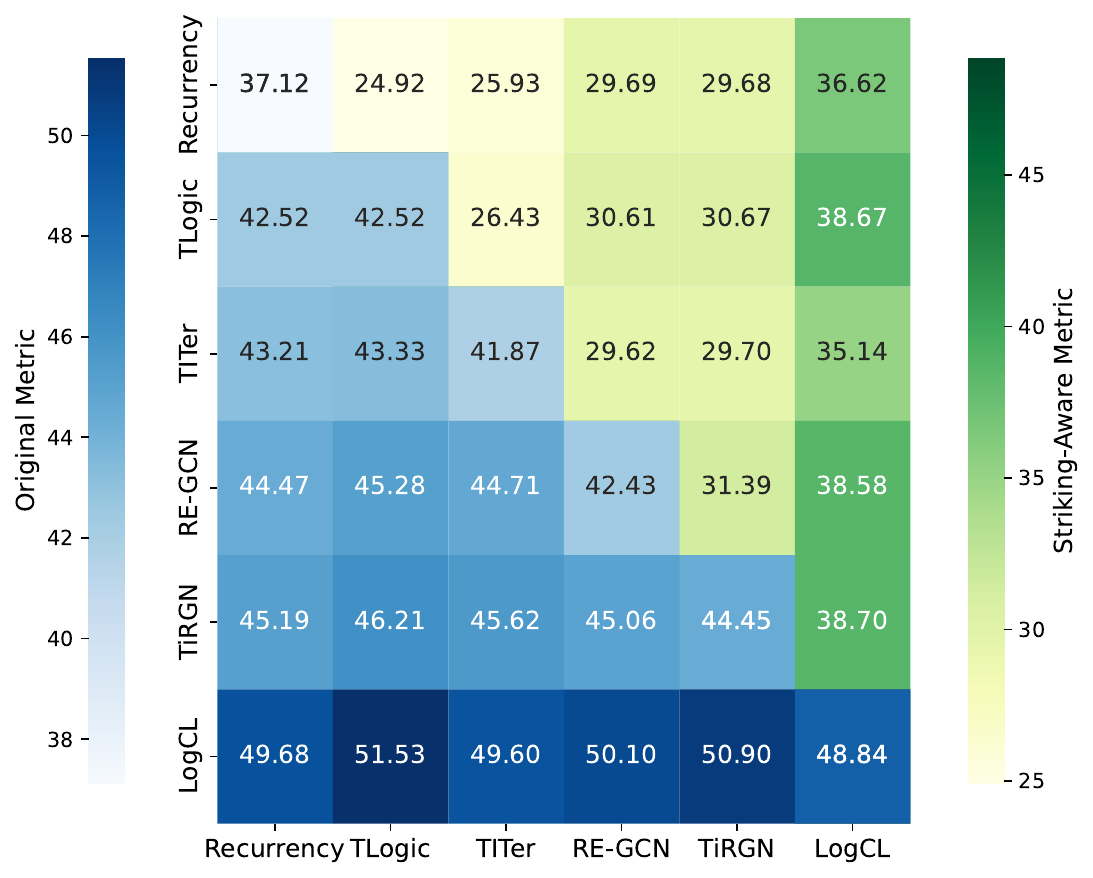}
        \end{minipage}
    }
    \subfigure[ICEWS18]{
        \begin{minipage}[t]{0.4\linewidth}
            \centering
            \includegraphics[width=1.05\linewidth]{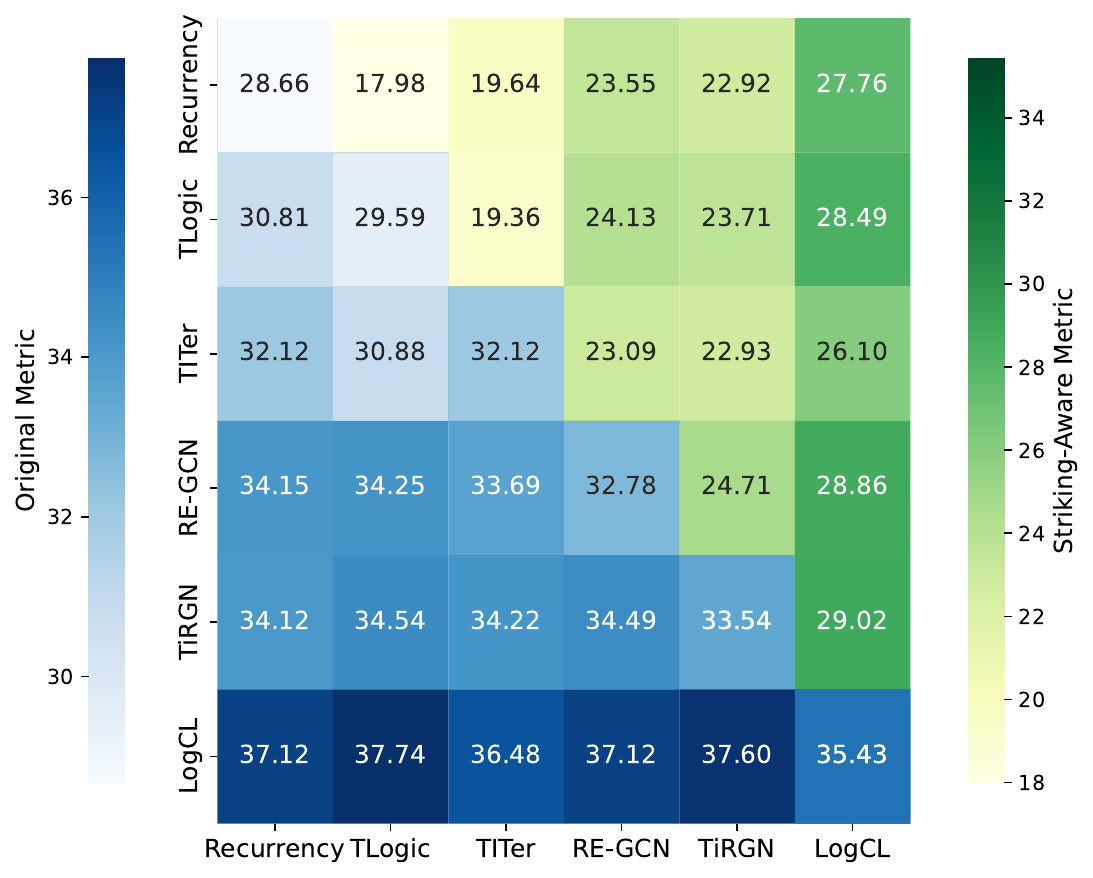}
        \end{minipage}
    }
    \subfigure[ICEWS05-15]{
        \begin{minipage}[t]{0.4\linewidth}
            \centering
            \includegraphics[width=1.05\linewidth]{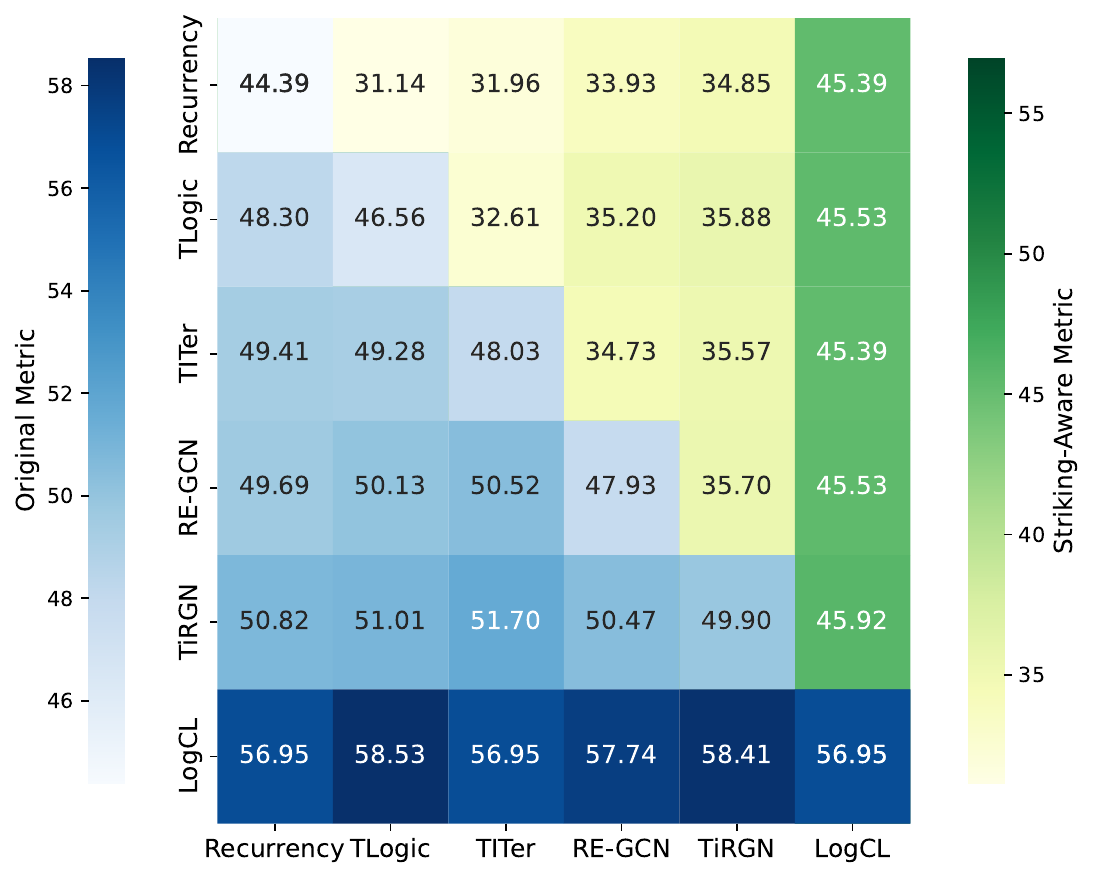}
        \end{minipage}
    }
    \subfigure[GDELT]{
        \begin{minipage}[t]{0.4\linewidth}
            \centering
            \includegraphics[width=1.05\linewidth]{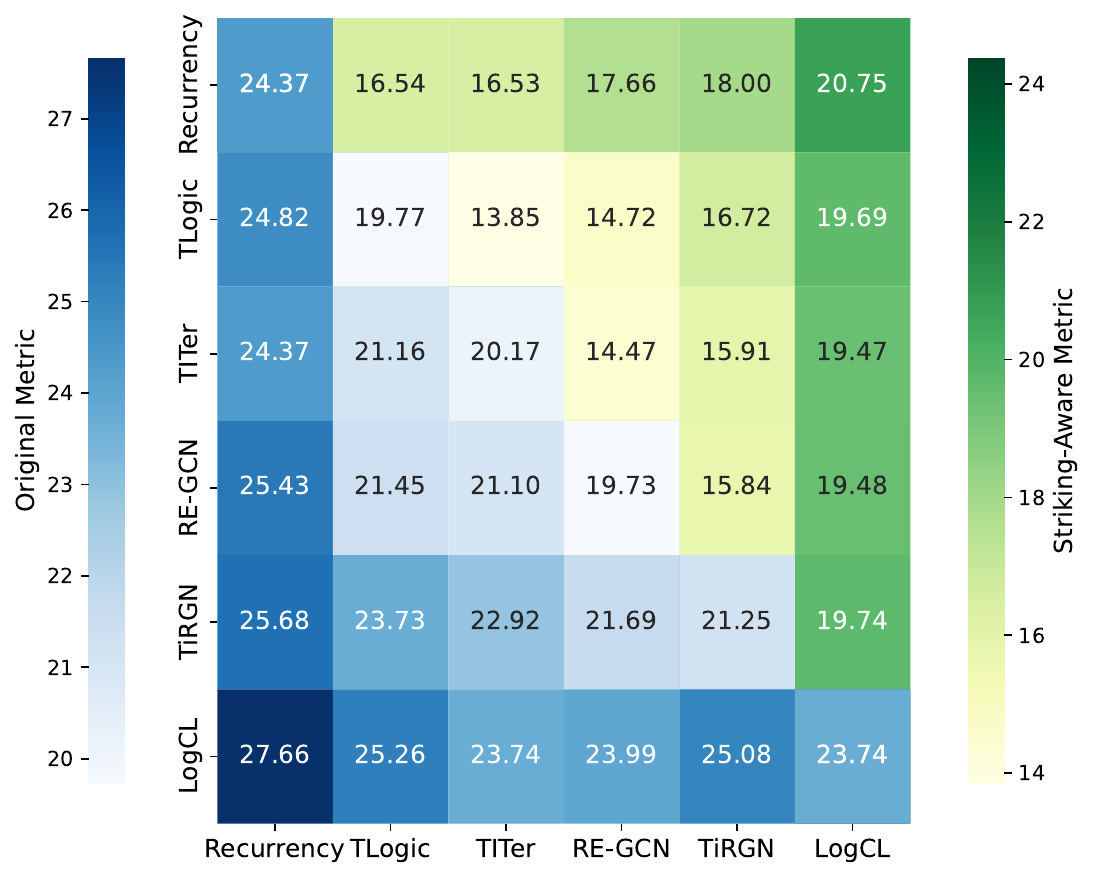}
        \end{minipage}
    }
    \caption{Ensemble results of different baseline models. The lower triangle (including the diagonal) represents the original metrics and the upper triangle corresponds to the striking-aware metrics}
    \label{ensemble_result}
\end{figure*}

\paragraph{Context Dependence, and Time Sensitivity of Outstanding Event} In Figure \ref{Dist Study}, we analyze the distribution of event strikingness under different parameter settings. The results demonstrate that the strikingness distribution changes with the variation of parameters, indicating that the proposed RSMF effectively selects the expected outstanding events by adjusting the specific parameters. Specifically, $\tau$ is the confidence threshold for constraining the number of applicable rules, $\alpha^s$ denotes the weight of the entity in strikingness calculation, and $w$ influences the window length of historical knowledge graphs. Adjustments to these parameters will affect the context information of future events, reflecting the context dependence of outstanding events. Similarly, $\lambda$ controls the temporal decay rate of rules, showcasing the time sensitivity. It should be emphasized that these parameters, unlike model hyperparameters, are not designed for achieving optimal model performance. Our goal is to establish a well-calibrated evaluation framework. The parameters for strikingness calculation in RSMF are instrumental in implementing its core definition and generating a long-tailed distribution of event strikingness (Figure \ref{Dist Study}). To facilitate reproducibility and promote community adoption, we provide standardized configurations as recommended defaults.

\subsection{Hyperparameter Sensitivity of Evaluation}
The construction of RSMF involves several hyperparameters. As shown in Figure \ref{Dist Study}, these hyperparameters can change the values and distribution of strikingness. Consequently, a natural concern arises: could these hyperparameters affect the outcomes of the strikingness‑aware evaluation framework, that is, causing a model A to outperform model B under one hyperparameter setting while underperforming under another? Figure \ref{Hyper Study} shows the influence of the involved hyperparameters on evaluation results on ICEWS14. It can be observed that although the WMRR values change, the relative performance ranking among models remains consistent throughout. This indicates that the strikingness-aware evaluation framework is robust to hyperparameter variations and can be reliably applied to assess TKGR models.

We also report the evaluation results with respect to parameter $b$ on the other datasets in Figure \ref{select_b_more}. As shown, the position at which evaluation results (i.e., model rankings) change varies across different datasets. At $b = 0.1$, both the WMRR and the relative model rankings remain comparatively stable, which is why we recommend this as the default.

\section{Ensemble Combinations}
\label{appendix:Ensemble Combinations}
We hypothesize that each baseline offers distinct predictive perspectives and thus explore various ensemble combinations for prediction. The results of all possible ensemble combinations are presented in Figure \ref{ensemble_result}, where the lower triangle (including the diagonal) represents the original metrics and the upper triangle corresponds to the striking-aware metrics. As observed, nearly all ensemble combinations achieve significant improvements on the original metrics. However, such gains are negligible on the striking-aware metrics. This is because ensemble methods primarily enhance the prediction of low-strikingness events, while for high-strikingness events, they may even introduce conflicts.

\end{document}